\documentclass[journal,11pt,draftclsnofoot,onecolumn]{IEEEtran}
\usepackage{amssymb}
\usepackage{pstricks}
\usepackage{amsmath}
\usepackage{subfigure}
\usepackage{enumerate}
\usepackage{tikz}
\usepackage{algorithm}
\usepackage{algpseudocode}
\usepackage{setspace} 
\usetikzlibrary{shapes,decorations}
\usepgflibrary{shapes.geometric}

\newtheorem{definition}{Definition}
\newtheorem{theorem}{Theorem}

\newtheorem{example}{Example}
\allowdisplaybreaks

\long\def\symbolfootnote[#1]#2{\begingroup%
 \def\thefootnote{\fnsymbol{footnote}}\footnote[#1]{#2}\endgroup}

\usepackage{cite}

\allowdisplaybreaks

\begin{document}

\def\Rd {{\mathbb {R}}^d}
\def\Rn {{\mathbb {R}}^N}
\def\Rm {{\mathbb {R}}^M}
\def\Nr {{\mathbb {N}}}
\def\Rmt {{\mathbb {R}}^{2M}}
\def\R {{\mathbb {R}}}
\def\Z {{\mathbb {Z}}}
\def\C {{\cal {C}}}
\def\Xb {{\mathbf{X}}}
\def\xb {{\mathbf{x}}}
\def\yb {{\mathbf{y}}}
\def\zb {{\mathbf{z}}}
\def\Mc {{\cal M}}
\def\eq#1{\begin{equation}{#1}\end{equation}}

\title{Graphical Models as Block-Tree Graphs}
\author{Divyanshu Vats and Jos\'{e} M. F. Moura
\thanks{The authors are with the Department of Electrical and Computer Engineering, Carnegie Mellon University, Pittsburgh, PA, 15213, USA (email: dvats@andrew.cmu.edu, moura@ece.cmu.edu, ph: (412)-268-6341, fax: (412)-268-3980).}
}
\date{}

\maketitle

\begin{abstract}
We introduce \emph{block-tree graphs} as a framework for deriving efficient algorithms on graphical models.  We define block-tree graphs as a tree-structured graph where each node is a cluster of nodes such that the clusters in the graph are disjoint.  This differs from junction-trees, where two clusters connected by an edge always have at least one common node.  When compared to junction-trees, we show that constructing block-tree graphs is faster and finding optimal block-tree graphs has a much smaller search space.  For graphical models with boundary conditions, the block-tree graph framework transforms the boundary valued problem into an initial value problem.  For Gaussian graphical models, the block-tree graph framework leads to a linear state-space representation.  Since exact inference in graphical models can be computationally intractable, we propose to use spanning block-trees to derive approximate inference algorithms.  Experimental results show the improved performance in using spanning block-trees versus using spanning trees for approximate estimation over Gaussian graphical models.
\end{abstract}


\pagebreak
\section{Introduction}

A graphical model is a random vector defined on a graph such that each node represents a random variable (or multiple random variables), and edges in the graph represent conditional independencies.  The underlying graph structure in a graphical model leads to a factorization of the joint probability distribution.  This property has lead to graphical models being used in many applications such as sensor networks, image processing, computer vision, bioinformatics, speech processing, and ecology \cite{Reu2005,WainwrightJordan2008}, to name a few.  
This paper derives efficient algorithms on graphical models.  
The structure of the graph plays an important role in determining the complexity of these algorithms.  Tree-structured graphs are suitable for deriving efficient inference and estimation algorithms \cite{Pearl1988}.  \emph{Inference} in graphical models corresponds to finding marginal distributions given a joint probability distribution.  \emph{Estimation} of graphical models corresponds to performing inference over the conditional distribution $p(\xb | \yb)$, where $\xb$ is a random vector defined on a graph with noisy observations $\yb$.  State-space models can be interpreted as graphical models defined on a chain or a tree \cite{ChouWillsky1994a, ChouWillsky1994b}, for which efficient estimation algorithms include the Kalman filter \cite{Kalman1960} or recursive smoothers \cite{KailathSayed}.  Estimation and inference in arbitrary chain or tree structured graphical models is achieved via belief propagation \cite{Pearl1988}.  These graphical models, however, have limited modeling capability \cite{SudderthWainwrightWillsky2004}, and it is thus desirable to consider more general graphs, i.e., graphs with cycles, an example of which is shown in Fig.~\ref{fig:ex_ug_1}.  

A popular method for inference in graphs with cycles is to perform variable elimination, where the joint probability distribution is marginalized according to a chosen elimination order, which is a permutation of the nodes in the graph.  Frameworks for variable elimination have been proposed in \cite{ZhangPoole1994,Dechter1999,Cozman2000}.  A general framework for variable elimination is achieved by constructing a junction-tree \cite{LauritzenSpiegelhalter1988}, which is a tree-structured graph with edges between clusters of nodes.  The key properties of junction-trees are highlighted as follows:

\begin{enumerate}[$(i)$]
\item \textbf{Clusters in a junction-tree:} Two clusters connected by an edge in a junction-tree always have at least one common node.  The number of nodes in the cluster with maximum size minus one is called the width of a graph, denoted as $\text{w}(G)$ for a graph $G$.

\item \textbf{Constructing junction-trees:} This consists of two steps: triangulation, which has complexity $O(n)$, and a maximum spanning tree algorithm, which has complexity $O(m^2)$, where $m$ is the number of cliques (see Section \ref{sec:review}) in a triangulated\footnote{A graph is triangulated if all cycles of length four or more have an edge connecting non-adjacent nodes in the cycle} graph \cite{JensenJenson1994}.
The number of cliques $m$ depends on the connectivity of the graph: if a graph is dense (many edges), $m$ can be small and if a graph is sparse (small number of edges), $m$ can be as large as $n-1$.

\item \textbf{Optimal junction-trees:} For a graph $G$, there can be many different associated junction-trees.  An \emph{optimal junction-tree} is the one with minimal width, called the treewidth of the graph \cite{Robertson1986}, denoted as $\text{tw}(G)$.  Finding the optimal junction-tree, and thus the treewidth of a graph, requires a search over at most $n!$ number of possible combinations, where $n$ is the number of nodes in a graph.

\item \textbf{Complexity of inference:} Inference in graphical models using junction-trees can be done using algorithms proposed in \cite{LauritzenSpiegelhalter1988, Jenson1990, ShaferShenoy1990}. 
The complexity of inference using junction-trees is exponential in the treewidth of the graph \cite{ShaferShenoy1990}.
\end{enumerate}

From the above analysis, it is clear that constructing junction-trees can be computationally difficult.  Further, finding optimal junction-trees is hard because of the large search space.
Since finding optimal junction-trees is hard, finding the treewidth of a graph is also hard \cite{Arnborg1987}.  Thus, the complexity of inference using junction-trees really depends on the upper bound on treewidth computed using heuristic algorithms, such as those given in \cite{Bodlaender2010}.

In this paper, we introduce \emph{block-tree graphs}, as an alternative framework for constructing tree-structured graphs from graphs with cycles.  The key difference between block-trees and junction-trees is that the clusters in a block-tree graph are disjoint, whereas clusters in a junction-tree have common nodes.  We use the term block-tree because the adjacency matrix for block-tree graphs is block-structured under a suitable permutation of the nodes.  The key properties of block-tree graphs and its comparison to the junction-tree are outlined as follows:

\begin{enumerate}[$(i')$]
\item \textbf{Clusters in a block-tree:} Clusters in a block-tree graph are always disjoint.  We call the number of nodes in the cluster with maximum size the block-width, denoted as $\text{bw}(G)$ for a graph $G$.

\item \textbf{Constructing block-trees:}
We show that a graph can be transformed into a block-tree graph by appropriately clustering nodes of the original graph.  The algorithm we propose for constructing block-tree graphs involves choosing a root cluster and finding successive neighbors.  An important property is that a block-tree graph is uniquely specified by the choice of the root cluster.  Thus, constructing block-tree graphs only requires knowledge of a root cluster, which is a small fraction of the total number of nodes in the graph.  On the other hand, constructing junction-trees requires knowledge of an elimination order, which is a permutation of all the nodes in the graph.  Constructing a block-tree graph has complexity $O(n)$, where $n$ is the number of nodes in the graphs.  When compared to junction-trees, we avoid the $O(m^2)$ computational step, which is significant savings when $m$ is as large as $n$.

\item \textbf{Optimal block-trees:}  Different choices of root clusters result in different block-tree graphs.  We define an \emph{optimal block-tree graph} as the block-tree graph  with minimal block-width, which we call the \emph{block-treewidth} of a graph, denoted as $\text{btw}(G)$.  We show that computing the optimal block-tree, and thus the block-treewidth, requires a search over $n \choose \lceil n/2 \rceil$ possible number of choices.  Although possibly very large for large $n$, this number is much less than $n!$, the search space for computing optimal junction-trees.

\item \textbf{Complexity of inference:}  We show that the complexity of using block-tree graphs for inference over graphical models is exponential in the maximum sum of cluster sizes of adjacent clusters.
\end{enumerate}

From $(i')-(iii')$, we see that constructing block-tree graphs is faster and finding optimal block-tree graphs has a smaller search space.  In general, the complexity of inference using block-tree graphs is higher, however, we show that there do exist graphical models for which the complexity of inference is the same for both the junction-tree and the block-tree graph.

Using disjoint clusters to derive efficient algorithms on graphical models has been considered in the past, but only in the context of specific graphical models.  For example, \cite{Woods1977} and \cite{MouraBalram1992} derive recursive estimators for graphical models defined on a 2-D lattice by scanning the lattice horizontally (or vertically).  For specific directed graphs, the authors in \cite{Cooper1984} and \cite{PengReggia1986}, used specific disjoint clusters for inference.  To our knowledge, previous work has not addressed questions like optimality of different structures or proposed algorithms for constructing tree-structured graphs using disjoint clusters.  Our block-tree graphs address these questions for any given graph, even non-lattice graphs and arbitrary directed graphs.

Applying our block-tree graph framework to undirected graphical models with boundary conditions, such that the boundary nodes connect to the undirected components in a directed manner, we convert a boundary valued problem into an initial value problem.  Motivation for using such graphs, which are referred to as chain graphs in the literature \cite{LauritzenWemuth1989,Frydenberg1990,AndersonMadiganPerlman2001}, is in accurately modeling physical phenomena whose underlying dynamics are governed by partial differential equations with local conditions imposed on the boundaries.  To not confuse chain structured graphs with chain graphs, in this paper we refer to chain graphs as boundary valued graphs. Such graphical models have been used extensively in the past to model images with boundary conditions being either Dirichlet, Neumann, or periodic, see \cite{BesagMoran1975,KashyapChellapa1983,ChellapaKashyap1985,LevyAdamsWillsky1990, MouraBalram1992} for examples.  To enable recursive processing, past work has either ignored the effect of boundaries or assumed simpler boundary values.  Using our block-tree graph framework, we cluster all boundary nodes in the chain graph into one cluster and then build the block-tree graph.  In \cite{VatsMoura2009j}, we derived recursive representations, which we called a telescoping representation, for random fields over continuous indices and random fields over lattices with boundary conditions.  The results presented here extend the telescoping representations to arbitrary boundary valued graphs, not necessarily restricted to boundary valued graphs over 2-D lattices.
Applying our block-tree graph framework to Gaussian graphical models, we get linear state-space representations, which leads to recursive estimation equations like the Kalman filter \cite{Kalman1960} or the Rauch-Tung-Striebel \cite{RauchTungStriebel1965} smoother.  

As mentioned earlier, the complexity of inference in graphical models is exponential in the treewidth of the graph.  Thus, inference in graphical models is computationally intractable when the treewidth is large \cite{Cooper1990}.  For this reason, there is interest in efficient approximate inference algorithms.  Loopy belief propagation (LBP), where we ignore the cycles in a graph and apply belief propagation, is a popular approach to approximate inference \cite{Pearl1988}. Although LBP works well in several graphs, convergence of LBP is not guaranteed, or the convergence rate may be slow \cite{MurphyWeissJordan1999,SudderthWainwrightWillsky2004}.  Another class of algorithms is based on decomposing a graph into several computationally tractable subgraphs and using the estimates on the subgraphs to compute the final estimate \cite{WainwrightMAP2002,WainwrightTRP2003,SudderthWainwrightWillsky2004}.  We show how block-tree graphs can be used to derive efficient algorithms for estimation in graphical models.  The key step is in using the block-tree graph to find subgraphs, which we call \emph{spanning block-trees}.  We apply the spanning block-tree framework to the problem of estimation in Gaussian graphical models and show the improved performance over spanning trees.

\textbf{Organization: }
Section \ref{sec:background} reviews graphical models, inference algorithms for graphical models, and the junction-tree algorithm.  
Section \ref{sec:block_tree_graph} introduces block-tree graphs, outlines an algorithm for constructing block-tree graphs given an arbitrary undirected graph, and introduces optimal block-tree graphs.  Section \ref{sec:inference_bt} outlines an algorithm for inference over block-tree graphs and discusses the computational complexity of such algorithms.  Section \ref{sub:boundary_valued_graphs} considers the special case of boundary valued graphs.  Section \ref{sec:gaussian_gm} considers the special case of Gaussian graphical models and derives linear recursive state-space representations, using which we outline an algorithm for recursive estimation in graphical models.  Section \ref{sec:approximate_estimation} considers the problem of approximate estimation of Gaussian graphical models by computing spanning block-trees.  Section \ref{sec:summary} summarizes the paper.

\section{Background and Preliminaries}
\label{sec:background}

Section \ref{sec:review} reviews graphical models.  For a more complete study, we refer to \cite{Lauritzen1996}.  Section \ref{sec:inference_algorithms} reviews inference algorithms for graphical models.

\subsection{Review of Graphical Models}
\label{sec:review}

Let $\xb = \{x_s \in \R^d : s \in V\}$ be a random vector defined on a graph $G = (V,E)$, where $V = \{1,2,\ldots,n\}$ is the set of nodes and $E \subset V \times V$ is the set of edges.  
Given any subset 
$W \subset V$, let $x_W = \{x_s : s \in W \}$ denote the set of random variables on $W$.  
An edge between two nodes $s$ and $t$ can either be directed, which refers to an edge from node $s$ to node $t$, or undirected, where the ordering does not matter, i.e., both $(s,t)$ and $(t,s)$ belong to the edge set $E$.  One way of representing the edge set is via an $n \times n$ \emph{adjacency matrix} $A$ such that
$A(i,j) = 1$ if $(i,j) \in E$,
$A(i,j) = 0$ if $(i,j) \notin E \,,$
where we assume $A(i,i) = 1$ for all $i = 1,\ldots,n$.  A \emph{path} is a sequence of nodes such that there is either an undirected or directed edge between any two consecutive nodes in the path.  A graph with only directed edges is called a directed graph.  A directed graph with no cycles, i.e., there is no path with the same start and end node, is called a directed acyclic graph (DAG).  A graph with only undirected edges is called an undirected graph.  Since DAGs can be converted to undirected graphs via moralization, see \cite{Lauritzen1996}, in this paper, unless mentioned otherwise, we only study undirected graphs.
For any $s \in V$, ${\cal N}(s) = \{ t \in V : (s,t) \in E\}$
defines the \emph{neighborhood} of $s$ in the undirected graph $G = (V,E)$.  The \emph{degree} of a node $s$, denoted $d(s)$, is the number of neighbors of $s$.  A set of nodes ${\cal C}$ in an undirected graph is a \emph{clique} if all the nodes in ${\cal C}$ are connected to each other, i.e., all nodes in ${\cal C}$ have an undirected edge.
A random vector $\xb$ defined on an undirected graph $G$ is referred to as an \emph{undirected graphical model} or a \emph{Markov random field}.  
The edges in an undirected graph are used to specify a set of conditional independencies in the random vector $\xb$.  For any disjoint subsets $A,B,C$ of $V$, we say that $B$ separates $A$ and $C$ if all the paths between $A$ and $C$ pass through $B$.  For undirected graphical models, the global Markov property is defined as follows:
\begin{definition}[Global Markov Property]
\label{def:global_markov_property}
For subsets $A,B,C$ of $V$ such that $B$ separates $A$ and $C$, $x_A$ is conditionally independent of $x_C$ given $x_B$, i.e., $x_A \perp x_C | x_B$.
\end{definition}

By the Hammersley-Clifford theorem, the probability distribution $p(\xb)$ of Markov models is factored in terms of cliques as \cite{Besag1974}
\begin{equation}
 p(\xb) = \frac{1}{Z} \prod_{C \in \C} \psi_C(x_C) \,, \label{eq:dist_mrf} 
\end{equation}
where $\{\psi_C(x_C)\}_{C \in \C}$ are positive potential functions, also known as clique functions, that depend only on the variables in the clique $C \in \C$, and $Z$, the partition function, is a normalization constant.  

Throughout the paper, we assume that a given graph is connected, which means that there exists a path between any two nodes of the graph.  If this condition does not hold, we can always split the graph into more than one connected graph and separately study each connected graph.
A \emph{subgraph} of a graph $G = (V,E)$ is graph with vertices and edges being a subset of $V$ and $E$, respectively.  In the next Section, we review algorithms for doing inference in graphical models.

\subsection{Inference Algorithms}
\label{sec:inference_algorithms}

Inference in graphical models corresponds to finding marginal distributions, say $p(x_s)$, given the joint probability distribution $p(\xb)$ for $\xb = \{x_1,\ldots,x_n\}$.  All inference algorithms derived on $p(\xb)$ can be applied to the problem of estimation, where we want to marginalize the joint distribution $p(\xb | \yb)$ to find $p(x_s | \yb)$, where $\yb$ is a noisy observation of the random vector $\xb$.

For tree-structured graphs, belief propagation \cite{Pearl1988} is an efficient algorithm for inference with complexity linear in the number of nodes.  For graphs with cycles, as discussed in Section I, a popular method is to first construct a junction-tree and then apply belief propagation \cite{LauritzenSpiegelhalter1988}.  We now consider two examples that will act as running examples throughout the paper.

\begin{example}
\label{example_1}
Consider the undirected graph in Fig. \ref{fig:ex_ug_1} and its junction-tree shown in Fig. \ref{fig:ex_ug_jt_1}.  The clusters in the junction-tree are represented as ellipses (these are the cliques in the triangulated graph producing the junction-tree).  On the edges connecting clusters, we have separator nodes that correspond to the common nodes connecting two clusters.  It can be shown that this junction-tree is optimal, and thus the treewidth of the graph in Fig. \ref{fig:ex_ug_1} is three.
\end{example}

\begin{example}
\label{example_2}
By deleting the edge between nodes $3$ and $5$ in Fig.~\ref{fig:ex_ug_1}, we get the undirected graph in Fig. \ref{fig:ex_ug_2}.  The optimal junction tree is shown in Fig. \ref{fig:ex_ug_jt_2} (the separator nodes are ignored for simplicity).  The treewidth of the graph is two.
\end{example}

{
\begin{figure}
\begin{center}
\subfigure[Undirected graph]{
\begin{tikzpicture}[scale=0.7]
\tikzstyle{every node}=[draw,shape=circle,scale=0.5];
\path (-1,1) node (x1) {$1$};
\path (-1,0) node (x2) {$2$};
\path (0,1) node (x3) {$3$};
\path (0,0) node (x4) {$4$};
\path (1,2) node (x5) {$5$};
\path (1,1) node (x6) {$6$};
\path (1,0) node (x7) {$7$};
\path (2,1) node (x8) {$8$};
\path (2,0) node (x9) {$9$};
\draw (x1) -- (x2) -- (x4) -- (x3) -- (x1);
\draw (x3) -- (x5) -- (x8) -- (x9) -- (x7) -- (x4) -- (x6) -- (x3);
\draw (x6) -- (x7);
\draw (x6) -- (x8);
\end{tikzpicture}
\label{fig:ex_ug_1}
} \qquad
\subfigure[Junction tree for (a).]{
\begin{tikzpicture}[scale=0.7]
\tikzstyle{every node}=[draw,scale=0.5];
\path (-1,1.6) node[ellipse] (c1) {$2$ $3$ $4$};
\path (-1,0) node[ellipse] (c2) {$3$ $4$ $6$ $8$};
\path (-1,0.8) node[rectangle] (c0) {$3$ $4$};
\path (-1,-0.8) node[rectangle] (c3) {$3$ $8$};
\path (-1,-1.6) node[ellipse] (c4) {$3$ $5$ $8$};
\path (2,0) node[ellipse] (c5) {$7$ $4$ $6$ $8$};
\path (0.5,0) node[rectangle] (c6) {$4$ $6$ $8$};
\path (0.5,1.6) node[rectangle] (c7) {$2$ $3$};
\path (2,1.6) node[ellipse] (c8) {$1$ $2$ $3$};
\path (2,-0.8) node[rectangle] (c9) {$7$ $8$};
\path (2,-1.6) node[ellipse] (c10) {$7$ $8$ $9$};
\draw (c1) -- (c0) -- (c2) -- (c3) -- (c4);
\draw (c2) -- (c6) -- (c5);
\draw (c1) -- (c7) -- (c8);
\draw (c5) -- (c9) -- (c10);
\end{tikzpicture}
\label{fig:ex_ug_jt_1}
} \qquad
\subfigure[Undirected graph]{
\begin{tikzpicture}[scale=0.7]
\tikzstyle{every node}=[draw,shape=circle,scale=0.5];
\path (-1,1) node (x1) {$1$};
\path (-1,0) node (x2) {$2$};
\path (0,1) node (x3) {$3$};
\path (0,0) node (x4) {$4$};
\path (1,2) node (x5) {$5$};
\path (1,1) node (x6) {$6$};
\path (1,0) node (x7) {$7$};
\path (2,1) node (x8) {$8$};
\path (2,0) node (x9) {$9$};
\draw (x1) -- (x2) -- (x4) -- (x3) -- (x1);
\draw (x5) -- (x8) -- (x9) -- (x7) -- (x4) -- (x6) -- (x3);
\draw (x6) -- (x7);
\draw (x6) -- (x8);
\end{tikzpicture}
\label{fig:ex_ug_2}
} \qquad
\subfigure[Junction tree for (c).]{
\begin{tikzpicture}[scale=0.7]
\tikzstyle{every node}=[draw,scale=0.5];
\path (-1,-1) node[ellipse] (c1) {$1$ $2$ $3$};
\path (-1,0) node[ellipse] (c2) {$2$ $3$ $4$};
\path (-1,1) node[ellipse] (c3) {$3$ $4$ $6$};
\path (0.7,1) node[ellipse] (c4) {$4$ $6$ $7$};
\path (2.4,1) node[ellipse] (c6) {$6$ $7$ $8$};
\path (2.4,0) node[ellipse] (c7) {$7$ $8$ $9$};
\path (2.4,-1) node[ellipse] (c5) {$5$ $8$};
\draw (c1) -- (c2) -- (c3) -- (c4) -- (c6) -- (c7) -- (c5);
\end{tikzpicture}
\label{fig:ex_ug_jt_2}
}
\caption{Undirected graphs and their junction-trees.}
\label{fig:ug_example_1}
\end{center}
\vspace{-0.3cm}
\end{figure}
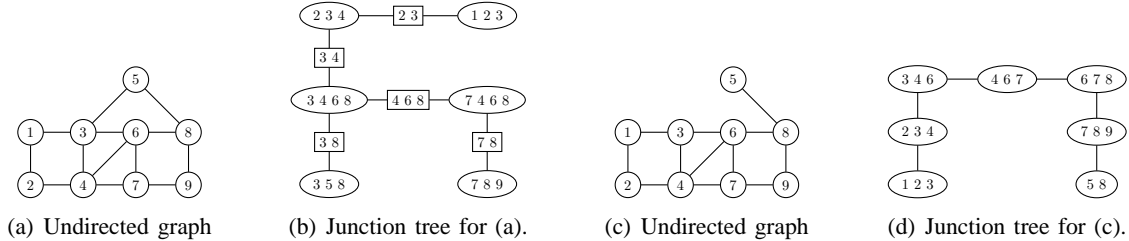
}

To do inference using junction-trees, we first associate potential functions with each clique.  This is done by grouping potentials from the original joint distribution and mapping them to their respective cliques. For example, in Example~\ref{example_1}, the junction tree has a clique $\{1,2,3\}$, so the clique function will be
$\Psi_{1,2,3} = \psi_{1,2}(x_1,x_2)\psi_{1,3}(x_1,x_3) \,,$
where $\psi_{1,2}(x_1,x_2)$ and $\psi_{1,3}(x_1,x_3)$ are factors in the original probability distribution.  Having defined potential functions for each clique, a message passing scheme, similar in spirit to belief propagation, can be formulated to compute marginal distributions for each clique \cite{LauritzenSpiegelhalter1988,ShaferShenoy1990}.  The marginal distribution of each node can be subsequently computed by marginalizing the distribution of the cliques.  The following theorem summarizes the time and space complexity of doing inference on junction trees.

\begin{theorem}[Complexity of inference using junction tree \cite{ShaferShenoy1990}]
\label{thm:complexity_jt}
For a random vector $\xb \in \R^n$ defined on an undirected graph $G = (V,E)$ with each $x_s$ taking values in $\Omega$, the time complexity for doing inference is exponential in the treewidth of the graph and the space complexity of doing inference is exponential in the treewidth of the graph plus one.
\end{theorem}

Theorem \ref{thm:complexity_jt} corresponds to the complexity of doing inference on an optimal junction tree.  However, finding the optimal junction-tree is hard in general, and thus the complexity is estimated by the upper bound of the treewidth of the graph, which can be found using algorithms in \cite{Bodlaender2010}.
The next Section introduces block-tree graphs as an alternative tree decomposition for graphs with cycles and shows that constructing block-tree graphs is less computationally intensive than constructing junction-trees and finding optimal block-trees has a smaller search space than finding optimal junction-trees.

\section{Block-Tree Graph}
\label{sec:block_tree_graph}
In this section, we introduce block-tree graphs and show the merits of using block-tree graphs over junction-trees.  
Section \ref{sec:definition_examples} defines a block-tree graph and gives examples.  Section \ref{sec:construction_block_tree} shows how to construct block-tree graphs starting from a connected undirected graph.
Section \ref{sec:optimal_block_tree} introduces optimal block-tree graphs.

\subsection{Definition and Examples}
\label{sec:definition_examples}

To define a block-tree graph, we first introduce block-graphs, which generalize the notion of graphs.
Throughout this paper, we denote block-graphs by ${\cal G}$ and graphs by $G$.  

\begin{definition}[Block-graph]
A \emph{block-graph} is the tuple ${\cal G} = ({\cal C},{\cal E})$, where ${\cal {C}} = \{C_1,C_2,\ldots,C_l\}$ is a set of disjoint clusters and ${\cal{E}}$ is a set of edges such that $(i,j) \in {\cal E}$ if there exists an edge between the clusters $C_i$ and $C_j$.
\end{definition}

Let the cardinality of each cluster be $\gamma_k = |C_k|$, and let $n$ be the total number of nodes.  If $\gamma_k = 1$ for all $k$, ${\cal G}$ reduces to an undirected graph.  For every block-graph ${\cal G}$, we associate a graph $G = (V,E)$, where $V$ is the set of all nodes in the graph and $E$ is the set of edges between nodes of the graph.  For each $(i,j) \in {\cal {\cal E}}$, the set of edges $E$ will contain at least one edge connecting two nodes in $C_i$ and $C_j$ or connecting nodes within $C_i$ or $C_j$.  A complete block-graph is defined as follows.

\begin{definition}[Complete block-graph]
For a block-graph ${\cal G} = ({\cal C},{\cal E})$, if all the nodes in $C_i$ have an edge between them, and for all $(i,j) \in E$ if all the nodes in $C_i$ and all the nodes in $C_j$ have an edge between them, then ${\cal G}$ is a complete block graph.
\end{definition}
We now introduce block-tree graphs.

\begin{definition}[Block-tree graph]
A block-graph ${\cal G} = ({\cal C},{\cal E})$ is called a block-tree graph if there exists only one path connecting any two clusters $C_i$ and $C_j$.
\end{definition}

Thus, block-tree graphs generalize tree-structured graphs.  Using Definition 3, we can define a complete block-tree graph.  As an example, consider the block-tree graph shown in Fig.~\ref{fig:ex_bt_1}, where $C_1 = \{1\}$, $C_2 = \{2,3\}$, $C_3 = \{4,5,6\}$, $C_4 = \{7,8\}$, and $C_5 = \{9\}$.  A complete block-tree graph corresponding to the block-tree in Fig.~\ref{fig:ex_bt_1} is shown in Fig.~\ref{fig:ex_ug_full}.  The block-tree graph in Fig.~\ref{fig:ex_bt_1} serves as a representation for a family of undirected graphs.  For example, Fig.~\ref{fig:ex_bt_1} serves as a representation for the graphs in Fig.~\ref{fig:ex_ug_1} and Fig.~\ref{fig:ex_ug_2}.  This can be seen by removing edges from Fig.~\ref{fig:ex_ug_full}.  In the next Section, we consider the problem of constructing a block-tree graph given an undirected graph.

\begin{figure}
\begin{center}
\subfigure[Block-tree]{
\begin{tikzpicture}[scale=0.7]
\tikzstyle{every node}=[draw,scale=0.5];
\path (-1,-1) node[ellipse] (c1) {$1$};
\path (-1,0) node[ellipse] (c2) {$2$ $3$};
\path (0.5,0) node[ellipse] (c3) {$4$ $5$ $6$};
\path (2.0,0) node[ellipse] (c4) {$7$ $8$};
\path (2.0,-1) node[ellipse] (c5) {$9$};
\draw (c1) -- (c2) -- (c3) -- (c4) -- (c5);
\end{tikzpicture}
\label{fig:ex_bt_1}
}
\subfigure[An undirected graph for (a)]{
\begin{tikzpicture}[scale=0.7]
\tikzstyle{every node}=[draw,scale=0.5];
\path (-2,0) node[circle] (c7) {$1$};
\path (-1,0.5) node[circle] (c4) {$3$};
\path (-1,-0.5) node[circle] (c8) {$2$};
\path (0,1) node[circle] (c1) {$5$};
\path (0,0) node[circle] (c5) {$6$};
\path (0,-1) node[circle] (c9) {$4$};
\path (1,0.5) node[circle] (c2) {$8$};
\path (1,-0.5) node[circle] (c6) {$7$};
\path (2,0) node[circle] (c3) {$9$};
\draw (c7) -- (c4) -- (c8) -- (c7);
\draw (c4) -- (c1) -- (c5) -- (c9) -- (c4);
\draw (c8) -- (c1) -- (c2) -- (c3) -- (c6);
\draw (c8) -- (c5);
\draw (c8) -- (c9);
\draw (c4) -- (c5) -- (c2) -- (c6) -- (c9) -- (c2);
\draw (c1) -- (c6) -- (c5);
\draw (c1) .. controls (0.6,0.6) and (0.6,-0.6) .. (c9);
\draw[color=white] (-3,0) -- (3,0);
\end{tikzpicture}
\label{fig:ex_ug_full}
}
\subfigure[Junction tree for Fig. \ref{fig:ex_ug_full}]{
\begin{tikzpicture}[scale=0.6]
\tikzstyle{every node}=[draw,scale=0.5];
\path (-2,-2) node[ellipse] (c1) {$1$ $2$ $3$};
\path (-2,-1) node[rectangle] (c1s) {$2$ $3$};
\path (-2,0) node[ellipse] (c2) {$2$ $3$ $4$ $5$ $6$};
\path (0,0) node[rectangle] (c2s) {$4$ $5$ $6$};
\path (2,0) node[ellipse] (c3) {$4$ $5$ $6$ $7$ $8$};
\path (2,-1) node[rectangle] (c3s) {$7$ $8$};
\path (2,-2) node[ellipse] (c4) {$7$ $8$ $9$};
\draw (c1) -- (c1s) -- (c2) -- (c2s) -- (c3) -- (c3s) -- (c4);
\end{tikzpicture}
\label{fig:ex_ug_full_jt}
}
\subfigure[Block-tree for Fig. \ref{fig:ex_ug_2}]{
\begin{tikzpicture}[scale=0.9]
\tikzstyle{every node}=[draw,scale=0.5];
\path (-2,-1.6) node[ellipse] (c1) {$1$};
\path (-2,-1) node[ellipse] (c2) {$2$ $3$};
\path (-1,-1) node[ellipse] (c3) {$4$ $6$};
\path (0,-1) node[ellipse] (c4) {$7$ $8$};
\path (0,-0.4) node[ellipse] (c5) {$5$};
\path (0,-1.6) node[ellipse] (c6) {$9$};
\draw (c1) -- (c2) -- (c3) -- (c4) -- (c5);
\draw (c4) -- (c6);
\draw[color=white] (-3,0) -- (1,0);
\end{tikzpicture}
\label{fig:ex_bt_2}
}
\caption{Example of block-trees}
\label{fig:bt}
\end{center}
\vspace{-0.3cm}
\end{figure}
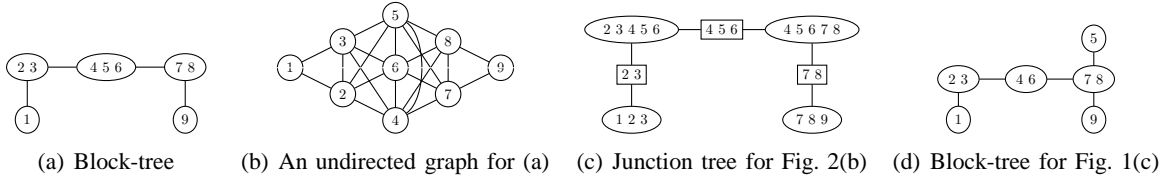

\subsection{Constructing Block-Tree Graphs}
\label{sec:construction_block_tree}

{
\begin{algorithm}
\caption{Constructing Block-tree Graphs}\label{alg:blocktree}
\begin{algorithmic}[1]
\Procedure{ConstructBlockTree}{$G,V_1$}
\State $r = 1$ ; $V_0 = \{\;\}$;
\While{$\bigcup V_r \ne V$}
\State $r = r + 1$
\State Find neighbors: $V_{r} = 
\{k : (j,k) \in E \; \forall\;,j \in V_{r-1}\} \backslash \{V_{r-2} \cup V_{r-1}\}$
\State Split Cluster: $\{V_r^{1}\,\ldots,V_r^{m_r}\}$ s.t. for all $i \in V_r^{n_1}$ and $j \in V_r^{n_2}$, $n_1 \ne n_2$, $(i,j) \notin E$.
\EndWhile
\For{$i = r,r-1,\ldots,3$}
\For{$j = 1,\ldots, m_i$}
\State Update cluster: Find $\{V_{i-1}^{j_1},\ldots,V_{i-1}^{j_w}\}$ s.t. there exists nodes $s_1,\ldots,s_{w}$, where $s_k \in V_{i-1}^{j_k}$, s.t. $(s_k,t) \in E$ for some $t \in V_{i}^j$.  Combine $\{V_{i-1}^{j_1},\ldots,V_{i-1}^{j_w}\}$ into one cluster and update $V_{i-1}$.
\EndFor
\EndFor
\State Relabel clusters as $C_1,\ldots,C_l$ and find edge set ${\cal E}$ s.t. $(i,j) \in {\cal E}$ if there exists an edge between $C_i$ and $C_j$.
\EndProcedure
\end{algorithmic}
\end{algorithm}
}

Our algorithm for constructing block-tree graphs is outlined in Algorithm \ref{alg:blocktree}.  The input to the algorithm is a connected graph $G$ and an initial cluster $V_1$, which we call the \emph{root cluster}.  The output of the algorithm is a block-tree graph ${\cal G} = ({\cal C},{\cal E})$.  The key steps of the algorithm are highlighted as follows:

\noindent
\textbf{Forward Pass:} (Lines 3-7) Starting from the root cluster $V_1$,
we iteratively find successive neighbors of $V_1$ to construct a sequence of $r$ clusters $V_1,V_2,\ldots,V_r$ such that $V_2 = {\cal N}(V_1) \backslash V_1$, $V_3 = {\cal N}(V_2) \backslash \{V_1 \cup V_2\},$ $\ldots,$ $V_r = {\cal N}(V_{r-1}) \backslash \{V_{r-2} \cup V_{r-1}\}$, where ${\cal N}(V_k)$ are the neighbors of the set of nodes in $V_k$.  This is shown in Line 5 of Algorithm \ref{alg:blocktree}.  For each $V_k$, $k = 2,\ldots, r$, we split $V_k$ into $m_k$ disjoint clusters $\{V_k^1,\ldots,V_k^{m_k}\}$ such that $\bigcup V_k^j = V_k$ and there are no edges between the clusters $V_1^i$ and $V_1^j$ for $i \ne j$.  This is shown in Line 6 of Algorithm \ref{alg:blocktree}.

\noindent
\textbf{Backwards Pass:} (Lines 8-13)  In this step, we find the final clusters of nodes given $V_1,V_2,\ldots,V_r$.  Starting at $V_r = \{V_r^1,\ldots,V_r^{m_r}\}$, for each $V_r^j$, $j = 1,\ldots,m_r$, we find all clusters $\{V_{r-1}^{j_1},\ldots,V_{r-1}^{j_w}\}$ such that there exists an edge between $V_{r-1}^{j_n}$, $n = 1,\ldots,w$ and $V_{r}^j$.  Combine $\{V_{r-1}^{j_1},\ldots,V_{r-1}^{j_w}\}$ into one cluster and then update the clusters in $V_{r-1}$ accordingly.  We repeat the above steps for all clusters $V_{r-1},\ldots, V_3$.  Thus, if $r = 2$, the backwards step is not needed.  Relabel all the clusters such that ${\cal C} = \{C_1,\ldots,C_l\}$ and find the edge set ${\cal E}$.

The forward pass of the algorithm first finds a chain structured block-graph over the clusters $V_1$,$V_2$,$\ldots,$ $V_r$.  The backwards pass then splits the clusters in $V_k$ to get a tree-structured graph.  The key intuition utilized in the backwards pass is that each cluster in $V_{k}$ connects to only one cluster in $V_{k-1}$.  If there are more than one such clusters in $V_{k-1}$, it is trivial to see that the resultant block-graph will have a cycle and will no longer be a block-tree graph.  

As an example, consider finding block-tree graphs for the graph in Fig.~\ref{fig:ex_ug_1}.  Starting with the root cluster $V_1 = \{1\}$, we have $V_2 = \{2,3\}$, $V_3 = \{4,5,6\}$, $V_4 = \{7,8\}$, and $V_5 = \{9\}$.  Further splitting the clusters $V_k$ and running the backwards pass, the clusters do not split, and we get the block-tree graph in Fig.~\ref{fig:ex_bt_1}.  We get the same block-tree graph if we start from the root cluster $V_1 = \{9\}$.  Now suppose, we start with the root cluster $V_1 = \{2,3\}$.  Then, we have $V_2 = \{1,4,5,6\}$, $V_3 = \{7,8\}$, and $V_4 = \{9\}$.  Splitting these clusters (Line 6 in Algorithm 1), we have $V_2^1 = \{1\}$, $V_2^2 = \{4,6\}$, $V_2^3 = \{5\}$, $V_3^1 = \{7\}$, $V_3^2 = \{8\}$, and $V_4^1 = \{4\}$.  Given these clusters, we now apply the backwards pass to find the final set of clusters:
\begin{enumerate}
\item The cluster $V_4^1 = \{4\}$ has edges in both $V_3^1 = \{7\}$ and $V_3^2 = \{8\}$, so we combine $V_3^1$ and $V_3^2$ to get $V_3^1 = \{7,8\}$.
\item Next, the cluster $V_3^1 = \{7,8\}$ has edges in $V_2^2$ and $V_2^3$, so we combine these clusters and get $V_2^1 = \{1\}$ and $V_2^{2} = \{4,5,6\}$.
\end{enumerate}
Given the above clusters, we get the same block-tree graph in Fig.~\ref{fig:ex_bt_1}.  The need of the backwards pass in Algorithm~\ref{alg:blocktree} is clear from the above example since it successfully splits the cluster $V_2$ with four nodes into two smaller clusters.  As another example, the block-tree graph for the graph in Fig.~\ref{fig:ex_ug_2} using a root cluster of $V_1 = \{1\}$ is shown in Fig.~\ref{fig:ex_bt_2}.

Notice that in Algorithm~\ref{alg:blocktree} we did not split the root cluster $V_1$.  Thus, one of the clusters in ${\cal C}$ will be $V_1$.  Without loss in generality, we assume $C_1 = V_1$.  We now show that Algorithm~\ref{alg:blocktree} always gives us a unique block-tree graph for each set of root cluster $V_1$.

\begin{theorem}
\label{thm:unique_bt}
Algorithm \ref{alg:blocktree} always outputs a block-tree graph ${\cal G} = ({\cal C},{\cal E})$ for each possible set of root cluster $V_1$ and undirected graph $G = (V,E)$, which is connected.  Further, the block-tree graph ${\cal G}$ is unique.
\end{theorem}
\begin{proof}
For the root cluster $V_1$, after the backwards pass of the algorithm, we have the set of clusters: $V_1, \{V_2^1,V_2^2,\ldots,V_2^{m_2} \}, \ldots,
\{V_r^1,V_r^2,\ldots,V_r^{m_r} \}$.
By construction, there are no edges between $V_k^{n_1}$ and $V_k^{n_2}$ for $n_1 \ne n_2$.  The total number of clusters is $l = 1+\sum_{k=2}^{r} m_k$.  For ${\cal G} = ({\cal C},{\cal E})$ to be a block-tree graph, the undirected graph $G = (\{1,2,\ldots,l\},{\cal E})$ to be a 
tree-structured graph.  For this, $G$ must be connected and the number of edges in the graph must be $|{\cal E}| = l-1$.  The block-tree graph formed using Algorithm \ref{alg:blocktree} is connected by construction since the original undirected graph $G = (V,E)$ is connected.  Counting the number of edges between clusters, we have $l-1$ edges, and thus the output of Algorithm \ref{alg:blocktree} is a block-tree graph.  The uniqueness of the block-tree graph follows from construction.
\end{proof}

The next theorem characterizes the complexity of constructing block-tree graphs.

\begin{theorem}[Complexity of Algorithm \ref{alg:blocktree}]
The complexity of constructing a block-tree graph is $O(n)$, where $n$ is the number of nodes in the graph.
\end{theorem}
\begin{proof}
The proof is trivial since the algorithm involves traversing the nodes of the graph.  We do this twice, once during the forward pass and once during the backwards pass.
\end{proof}

\noindent
\textbf{Comparison to junction-trees:}  As mentioned before, the key difference between block-trees and junction-trees is that block-trees are constructed using disjoint clusters, whereas clusters in a junction-tree have common nodes.  Constructing block-trees is computationally more efficient since constructing junction-trees requires an additional complexity of $O(m^2)$, where $m$ can be as large as $n$ for sparse graphs.  From Algorithm~\ref{alg:blocktree} and Theorem~\ref{thm:unique_bt}, we note that a block-tree graph is uniquely specified using a root cluster, which is a small number of nodes.  On the other hand, specifying a junction-tree requires an elimination order, the size of which can be as large\footnote{The exact size of the elimination order depends on the connectivity of the graph} as $n$.

\subsection{Optimal Block-Tree Graphs}
\label{sec:optimal_block_tree}

In this Section, we consider the problem of finding optimal block-tree graphs.  In order to define an optimal block-tree graph, we first introduce the notion of block-width and block-treewidth.

\begin{definition}[Block-width]
\label{def:block_width}
For an undirected graph $G = (V,E)$, the \emph{block-width} of the graph with respect to a root cluster $V_1$, $\text{bw}(G,V_1)$, is the maximum cluster size in the block-tree graph constructed using $V_1$ as the root cluster: 
$\text{bw}(G,V_1) = \max_{k} |\gamma_k| \,.$
\end{definition}
\begin{definition}[Block-treewidth]
\label{def:btw}
The \emph{block-treewidth}, $\text{btw}(G)$, of an undirected graph 
$G = (V,E)$ is the minimal block-width of a graph $G$ with respect to all root clusters:
$\text{btw}(G) = \min_{V_1 \subset V} \text{bw}(G,V_1) \,.$
\end{definition}

For example, $\text{bw}\left(G,\{1\}\right) = 3$ for the graph in Fig. \ref{fig:ex_ug_1}.  By checking over all possible root clusters, it is easy to see that the block-treewidth for the graph in Fig.~\ref{fig:ex_ug_1} is also three.  For Fig. \ref{fig:ex_ug_2}, $\text{bw}(G,\{1\}) = 2$ which is also the block-treewidth of the graph.  We can now define an optimal block-tree graph.

\begin{definition}[Optimal block-tree graph]
\label{def:optimal_bt}
A block-tree graph for an undirected graph $G = (V,E)$ with respect to a root cluster $V_1$ is optimal if the block-width with respect to $V_1$ is equal to the block-treewidth of the graph, i.e., $\text{bw}(G,V_1) = \text{btw}(G)$.
\end{definition}

We show in Section \ref{sec:spanning_bt} that the notion of optimality for block-tree graphs in Definition \ref{def:optimal_bt} is useful when finding spanning block-trees, which are subgraphs with lower block-treewidth.  Computing the block-treewidth of a graph requires a search over all possible root clusters, which has complexity of $O(2^n)$.  This search space can be simplified since if we choose a $V_1$ such that $|V_1| \ge \lceil n/2 \rceil$, the block-width of the graph will be $V_1$ itself.  Thus, the search space can be restricted to root clusters of length $\lceil n/2 \rceil$, which requires a search over $n \choose n/2$ number of possible clusters.  In comparison, computing the treewidth requires finding an optimal elimination order, which requires a search over $n!$ possible number of combinations.  Since 
$n! \gg {n \choose \lceil n/2 \rceil}$, the search space of computing the treewidth is much larger than the search space of computing the block-treewidth.  However, the problem of computing the block-treewidth is still computationally intractable as the search space grows exponentially as $n$ increases.

We now propose a simple heuristic to find an upper bound on the block-treewidth of a graph.  Let $G = (V,E)$ be an undirected graph with $n$ nodes.  Instead of searching over all possible root clusters, whose maximum size can be $\lceil n/2 \rceil$, we restrict the search space to smaller root clusters.  For $n$ small, we find the best root cluster of size two, and for $n$ large we find the best root cluster of size one.  Given the initial choice of the root cluster, we add nodes to this to see if the block-width can be lowered further.  For example, if the initial root cluster is $V_1$, we check over all $k \in V \backslash V_1$ and see if $\{V_1,k\}$ leads to a lower block-width.  We repeat this process until the block-width does not decrease further.  For small $n$, the complexity of this heuristic algorithm is $O(n^2)$, since we initially search over all clusters of size two.  For large $n$, the complexity is $O(n)$ since we search only over clusters of size one.  Table~\ref{table:tw_vs_btw} compares upper bounds on the treewidth vs. upper bounds on the block-treewidth for some standard graphs used in the literature\footnote{The graphs were obtained from the database in people.cs.uu.nl/hansb/treewidthlib/}.  
The upper bound on the treewidth is computed using a software package\footnote{See www.treewidth.com}.

\begin{table}
\caption{Upper bound on block-treewidth vs upper bound on treewidth} 
\centering 
\begin{tabular}{c c c c c} 
\hline\hline 
Graph & Treewidth & Block-treewidth & nodes & edges \\ [0.5ex]
\hline 
ship-ship-pp & 8 & 8 & 30 & 77 \\ 
water & 10 & 8 & 32 & 123 \\
fungiuk & 4 & 4 & 15 & 36 \\
pathfinder-pp & 7 & 6 & 12 & 43 \\ 
1b67 & 17 & 16 & 68 & 559 \\
1bbz & 28 & 23 & 57 & 543 \\
1bkb & 34 & 29 & 131 & 1485\\
1bkf & 39 & 37 & 106 & 1264\\
1bx7 & 11 & 11 & 41 & 195 \\
1en2 & 17 & 16 & 69 & 463\\
1on2 & 40 & 34 & 135 & 1527\\
$n \times n$ grid graph & $n$ & $n$ & $n^2$ & $2n(n-1)$\\ [1ex]
\hline \hline 
\end{tabular}
\label{table:tw_vs_btw} 
\end{table}

\section{Inference Using Block-Tree Graphs}
\label{sec:inference_bt}

In this Section, we outline an algorithm for inference in undirected graphical models using block-tree graphs. 
The algorithm is similar to belief propagation with the difference that message passing happens between clusters of nodes instead of individual nodes.  Let $\xb \in \R^n$ be a random vector defined on an undirected graph $G = (V,E)$.  Using $V_1$ as a root cluster, suppose we construct the block-tree graph ${\cal G} = ({\cal C},{\cal E})$ using Algorithm \ref{alg:blocktree}, where ${\cal C} = [C_1,C_2,\ldots,C_l]$ and $\gamma_k = |C_k|$.  From (\ref{eq:dist_mrf}), we know that $p(\xb)$ admits the factorization over cliques such that
\begin{equation}
p(\xb) = \frac{1}{Z} \prod_{C \in \C} \psi_C(x_C) \,, \label{eq:dist_mrf_1}
\end{equation}
where $\C$ is the set of cliques.
Using the block-tree graph, we can express the factorization of $p(\xb)$ as
\begin{equation}
p(\xb) = \frac{1}{Z} \prod_{(i,j) \in {\cal E}} \Psi_{i,j}\left(x_{C_i},x_{C_j}\right) \,,
\end{equation}
where the factors $\Psi_{i,j}\left(x_{C_i},x_{C_j}\right)$ correspond to a product of potential functions taken from the factorization in (\ref{eq:dist_mrf_1}), where each $\psi_C(x_C)$ is mapped to a unique $\Psi_{i,j}\left(x_{C_i},x_{C_j}\right)$.

As an example, consider a random vector $\xb \in \R^9$ defined on the graphical model in Fig. \ref{fig:ex_ug_2}.  The block-tree graph is given in Fig. \ref{fig:ex_bt_2} such that $C_1 = \{1\}$, $C_2 = \{2,3\}$, $C_3 = \{4,6\}$, $C_4 = \{7,8\}$, $C_5 = \{5\}$, and $C_6 = \{9\}$.  Using Fig. \ref{fig:ex_ug_2}, the joint probability distribution can be written as
\begin{equation}
p(\xb) = \psi_{1,2}\psi_{1,3} \psi_{2,4} \psi_{3,4}\psi_{3,6} 
\psi_{4,6}
\psi_{4,7}\psi_{6,7}\psi_{6,8}\psi_{7,9}\psi_{8,9} \psi_{8,5} \,, \label{eq:xb_joint_1}
\end{equation}
where we simplify $\psi_{i,j}(x_i,x_j)$ as $\psi_{i,j}$.
We can rewrite (\ref{eq:xb_joint_1}) in terms of the block-tree graph as
\[p(\xb) = \Psi_{1,2}(x_{C_1},x_{C_2})
\Psi_{2,3}(x_{C_2},x_{C_3})
\Psi_{3,4}(x_{C_3},x_{C_4}) 
\Psi_{4,5}(x_{C_4},x_{C_5})
\Psi_{4,6}(x_{C_4},x_{C_6}) \,,\]
where\;\;
$\Psi_{1,2}(x_{C_1},x_{C_2}) = \psi_{1,3}\psi_{1,2}$,\;\;\;\;
$\Psi_{2,3}(x_{C_2},x_{C_3}) = \psi_{3,6}\psi_{2,4}\psi_{4,6}$,\;\;\;\;
$\Psi_{2,4}(x_{C_3},x_{C_4}) = \psi_{6,8}\psi_{4,7}\psi_{6,7}$,\\
$\Psi_{4,5}(x_{C_4},x_{C_5}) = \psi_{8,5},$ and
$\Psi_{4,6}(x_{C_4},x_{C_6}) = \psi_{7,9}\psi_{8,9}$.
Since the block-tree graph is a tree decomposition, all algorithms valid for tree-structured graphs can be directly applied to block-tree graphs.  Thus, we can now use the belief propagation algorithm discussed in Section \ref{sec:inference_algorithms} to do inference on block-tree graphs.  The steps involved are similar, with an additional step to marginalize the joint distributions over each cluster:

\begin{enumerate}
\item For any cluster, say $C_1$, identify its leaves.
\item Starting from the leaves, pass messages along each edge until we reach the root cluster $C_1$:
\begin{equation}
m_{i \rightarrow j}(x_{C_j}) = \sum_{x_{C_i}} \Psi_{i,j}\left(x_{C_i},x_{C_j}\right)
\prod_{e \in {\cal N}(i) \backslash j} m_{e \rightarrow i}(x_{C_i}) \,, \label{eq:rrr}
\end{equation}
where ${\cal N}(i)$ is the neighboring cluster of $C_i$ and 
$m_{i \rightarrow j}(x_{C_j})$ is the message passed from cluster $C_i$ to $C_j$.
\item Once all messages have been communicated from the leaves to the root, pass messages from the root back to the leaves.
\item After the messages reach the leaves, the joint distribution for each cluster is given as
\begin{equation}
p(x_{C_i}) = \prod_{j \in {\cal N}(i)} m_{j \rightarrow i} (x_{C_i}) \,.
\end{equation}
\item To find the distribution of each node, marginalize $p(x_{C_i})$.
\end{enumerate}

We now analyze the complexity of doing inference using block-tree graphs.  Assume $x_s \in \Omega$, where $s \in V$ such that $|\Omega| = K$.  From (\ref{eq:rrr}), to pass a message from $C_i$ to $C_j$, for each $x_{C_j} \in \Omega^{|C_j|}$, we require $K^{|C_i|}$ number of additions.  Thus, this step will require $K^{|C_i| + |C_j|}$ number of additions since we need to compute $m_{i \rightarrow j}(x_{C_j})$ for all possible values $x_{C_j}$ takes.  Thus, the complexity of inference is given as follows:

\begin{theorem}[Complexity of inference using block-tree graphs]
\label{thm:complexity_bt}
For a random vector $\xb \in \R^n$ defined on an undirected graph $G = (V,E)$ with each $x_s$ taking values in $\Omega$, the complexity of performing inference is exponential in the maximum sum of cluster sizes of adjacent clusters.
\end{theorem}

Another way to realize Theorem \ref{thm:complexity_bt} is to form a junction-tree using the block-tree graph.  It is clear that, in general, using block-tree graphs for inference is computationally less efficient than using junction-trees.  However, for complete block-graphs, an example of which is shown in Fig. \ref{fig:ex_ug_full}, we see that both the junction-tree and the block-tree graph have the same computational complexity.  Thus, complete block-graphs are attractive graphical models to use the block-tree graph framework when the goal is to do exact inference.
In Section \ref{sec:approximate_estimation}, we illustrate the advantage of using the block-tree graph framework on arbitrary graphical models in the context of approximate inference in graphical models.

\section{Boundary Valued Graphs}
\label{sub:boundary_valued_graphs}

In this Section, we specialize our block-tree graph framework to boundary valued graphs, which are known as chain graphs in the literature.
We are only concerned with boundary valued graphs where we have one set of boundary nodes connected in a directed manner to nodes in an undirected graph.  The motivation for using these particular types of boundary valued graphs is in modeling physical phenomena whose underlying statistics are governed by partial differential equations satisfying boundary conditions \cite{LevyAdamsWillsky1990,MouraGoswami1997}.  A common example is in texture modeling, where the boundary edges are often assumed to satisfy either periodic, Neumann, or Dirichlet boundary conditions \cite{MouraBalram1992}.  If the boundary values are zero, the graph will become undirected, and we can then use the block-tree graph framework in Section \ref{sec:block_tree_graph}.

Let $G = \left(\{V,\partial V\},E_V^- \cup E_{\partial V}^- \cup E^{\rightarrow}\right)$ be a boundary valued graph, where $V^-$ is a set of nodes, called \emph{interior nodes}, and $\partial V$ is another set of nodes referred to as \emph{boundary nodes}.  We assume that the nodes in $V^-$ and $\partial V$ are connected by undirected edges, denoted by the edge sets $E_V^-$ and $E_{\partial V}^-$, and that there exist directed edges between the nodes of $V^-$ and $\partial V$.  To construct a block-tree graph, we first need to chose a cluster $C_1$.  As discussed in Section \ref{sec:block_tree_graph}, any choice of clusters can lead to a recursive representation; however, it is natural for boundary valued graphs to initiate at the boundary.  For this reason, we let the root cluster be $\partial V$ and then use Algorithm \ref{alg:blocktree} to construct a block-tree graph.  By choosing the boundary values first, we convert a boundary valued problem into an initial valued problem.  

An example of a boundary valued graph and its block-tree graph is shown in Fig. \ref{fig:chain_graph_example}.  We let $C_1$ be the boundary nodes $\{a,b,c,d\}$ and subsequently construct $C_2 = \{1,3,7,9\}$, $C_3 = \{2,4,6,8\}$, and $C_4 = \{5\}$ so that we get the chain structured graph in Fig. \ref{fig:chain_graph_example}(b).  The probability distribution of $\xb$ defined on this graph can be written as
{ \vspace{-0.5cm}\singlespace \begin{align}
p(\xb) &= P(a)P(b)P(c)P(d) \psi_{abcd}\psi_{1,a}\psi_{3,b}\psi_{7,c}\psi_{9,d}
\psi_{1:9} \,, \text{where,} \\
\psi_{1:9} &= \psi_{12} \psi_{14}\psi_{2,5}\psi_{2,3}
\psi_{3,6} \psi_{5,6} \psi_{4,5} \psi_{6,9} \psi_{8,9} \psi_{5,8}
\psi_{7,8} \psi_{4,7} \,, \nonumber \\
\psi_{abcd}^{-1} &= \sum_{1:9}\psi_{1,a}\psi_{3,b}\psi_{7,c}\psi_{9,d} \psi_{1:9} \label{eq:abcd}\,.
\end{align}}
Using the block-tree graph, we write the probability distribution as
{ \vspace{-0.5cm} \singlespace \begin{align}
p(\xb) &= \Psi_{1,2}(x_{V_1},x_{V_2}) \Psi_{2,3}(x_{V_2},x_{V_3})
\Psi_{3,4}(x_{V_3},x_{V_3}) \,, \text{where,} \label{eq:cgg} \\
\Psi_{1,2}(x_{V_1},x_{V_2}) &= P(a)P(b)P(c)P(d)\psi_{abcd}
\psi_{1,a}\psi_{3,b}\psi_{7,c}\psi_{9,d} \\
\Psi_{2,3}(x_{V_2},x_{V_3}) &= \psi_{12}\psi_{1,4}\psi_{3,2}\psi_{3,6}
\psi_{7,4}\psi_{7,8} \psi_{6,9} \psi_{8,9} \\
\Psi_{3,4}(x_{V_3},x_{V_3}) &= \psi_{2,5}\psi_{4,5}\psi_{6,5}\psi_{8,5}
\end{align}}
Notice that (\ref{eq:cgg}) did not require the calculation of a normalization constant $Z$.  This calculation is hidden in the potential function $\psi_{abcd}$ given by (\ref{eq:abcd}).  We note that the results presented here extend our results for deriving recursive representations for Gaussian lattice models with boundary conditions in \cite{VatsMoura2009j} to arbitrary (non-Gaussian) undirected graphs with boundary values.

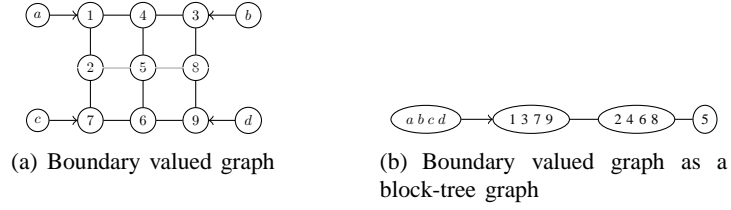
\begin{figure}
\begin{center}
\subfigure[Boundary valued graph]{
\begin{tikzpicture}[scale=0.7]
\tikzstyle{every node}=[draw,shape=circle,scale=0.5];
\path (-1,1) node (x1) {$1$};
\path (0,1) node (x2) {$4$};
\path (1,1) node (x3) {$3$};
\path (-1,0) node (x4) {$2$};
\path (0,0) node (x5) {$5$};
\path (1,0) node (x6) {$8$};
\path (-1,-1) node (x7) {$7$};
\path (0,-1) node (x8) {$6$};
\path (1,-1) node (x9) {$9$};
\path(-2,1) node (a) {$a$};
\path(2,1) node (b) {$b$};
\path(-2,-1) node (c) {$c$};
\path(2,-1) node (d) {$d$};
\draw[->] (a) -- (x1);
\draw[->] (b) -- (x3);
\draw[->] (c) -- (x7);
\draw[->] (d) -- (x9);
\draw (x1) -- (x2) -- (x3) -- (x6) -- (x9) -- (x8) -- (x7);
\draw (x7) -- (x4) -- (x1) (x4) -- (x5) -- (x6);
\draw (x2) -- (x5) -- (x8);
\draw[color=white] (-3,0) -- (3,0);
\end{tikzpicture}
} \qquad
\subfigure[Boundary valued graph as a block-tree graph]{
\begin{tikzpicture}[scale=0.7]
\tikzstyle{every node}=[draw,shape=circle,scale=0.5];
\path (-2,0) node[ellipse] (a) {$a \; b \; c \; d$};
\path (0,0) node[ellipse] (x1) {1 3 7 9};
\path (2,0) node[ellipse] (x2) {2 4 6 8};
\path (3.3,0) node[ellipse] (x3) {5};
\draw[->] (a) -- (x1);
\draw (x1) -- (x2) -- (x3);
\end{tikzpicture}
}
\end{center}
\caption{An example of a boundary valued graph}
\label{fig:chain_graph_example}
\vspace{-0.3cm}
\end{figure}

\section{Gaussian Graphical Models}
\label{sec:gaussian_gm}

In this Section, we specialize our block-tree graph framework to Gaussian graphical models.  Section \ref{sec:preliminary} reviews Gaussian graphical models and introduces relevant notations.  Section \ref{sec:state_space} derives linear state-space representations for undirected Gaussian graphical models.  Using these state-space representations, we derive recursive estimators in Section \ref{sec:recursive_estimation}.

\subsection{Preliminaries}
\label{sec:preliminary}

Let $\xb \in \R^n$ be a Gaussian random vector defined on a graph $G = (V,E)$, $V = \{1,2,\ldots,n\}$, and\footnote{For simplicity, we assume $x_k \in \R$, however, our results can be easily generalized when $x_k \in \R^d$, $d \ge 2$.} $x_k \in \R$ .  Without loss in generality, we assume that $\xb$ has zero mean and covariance $\Sigma$.
From \cite{Reu2005}, it is well known that the inverse of the covariance is sparse and the nonzero patterns in $J = \Sigma^{-1}$ determine the edges of the graph $G$.  In the literature, $J$ is often referred to as the information matrix or the potential matrix. 

\begin{figure}
\begin{center}
\subfigure[Scales]{
\scalebox{0.8} 
{
\begin{pspicture}(0,-2.6)(6.121406,2.64)
\usefont{T1}{ptm}{m}{n}
\rput(2.9623437,0.55){$C_k$}
\psline[linewidth=0.04cm](3.0264063,0.92156255)(3.6864061,2.1215625)
\psline[linewidth=0.04cm](2.6664062,0.26156256)(1.6864063,-0.7568142)
\psline[linewidth=0.04cm](2.8064063,0.26156256)(2.5864062,-0.7784375)
\psline[linewidth=0.04cm,linestyle=dotted,dotsep=0.16cm](2.9464064,0.24077936)(3.1464062,-0.7223437)
\psline[linewidth=0.04cm,linestyle=dotted,dotsep=0.16cm](3.0664062,0.24077936)(3.6864061,-0.7223437)
\psline[linewidth=0.04cm](3.3064063,0.24077936)(4.206406,-0.7584375)
\usefont{T1}{ptm}{m}{n}
\rput(4.1,2.5){$C_{\Upsilon(k)}$}
\usefont{T1}{ptm}{m}{n}
\rput(1.5823437,-1.0484375){$C_{\Gamma_1(k)}$}
\usefont{T1}{ptm}{m}{n}
\rput(2.7009375,-1.0484375){$C_{\Gamma_2(k)}$}
\usefont{T1}{ptm}{m}{n}
\rput(4.5009375,-1.0484375){$C_{\Gamma_q(k)}$}
\psline[linewidth=0.04cm](1.2064062,-1.3184375)(0.22640625,-2.3368142)
\psline[linewidth=0.04cm](1.3464062,-1.3184375)(1.1264062,-2.3584375)
\psline[linewidth=0.04cm,linestyle=dotted,dotsep=0.16cm](1.4864062,-1.3392206)(1.6864063,-2.3023436)
\psline[linewidth=0.04cm,linestyle=dotted,dotsep=0.16cm](1.6064062,-1.3392206)(2.2264063,-2.3023436)
\psline[linewidth=0.04cm](1.8464062,-1.3392206)(2.7464063,-2.3384376)
\psline[linewidth=0.04cm,linestyle=dotted,dotsep=0.10583334cm](3.7,0.96)(3.9,2.14)
\psline[linewidth=0.04cm,linestyle=dotted,dotsep=0.10583334cm](4.28,0.9)(4.2,2.2)
\psline[linewidth=0.04cm](4.96,0.96)(4.52,2.12)
\psframe[linewidth=0.02,linestyle=dashed,dash=0.16cm 0.16cm,dimen=outer](5.2,0.88)(0.0,-2.6)
\usefont{T1}{ptm}{m}{n}
\rput(0.7,1.1){${\cal T}(C_k)$}
\end{pspicture} 
}}
\subfigure[Block-tree]{
\begin{tikzpicture}[scale=0.7]
\tikzstyle{every node}=[draw,scale=0.6];
\path (0,1) node[ellipse] (c1) {$7$};
\path (0,0) node[ellipse] (c2) {$4$ $8$};
\path (-1,-1) node[circle] (c3) {$1$};
\path (1,-1) node[ellipse] (c4) {$5$ $9$};
\path (1,-2) node[ellipse] (c5) {$2$ $6$};
\path (1,-3) node[ellipse] (c6) {$3$};
\draw (c1) -- (c2) -- (c4) -- (c5) -- (c6);
\draw (c2) -- (c3);
\end{tikzpicture}
\label{fig:ex_scale}
}
\end{center}
\caption{Block-tree graph and related shift operators.}
\label{fig:scale_tree}
\vspace{-0.3cm}
\end{figure}
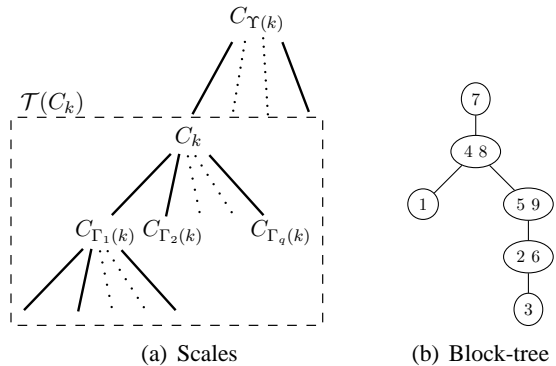

Suppose we construct a block-tree graph ${\cal G} = (\C,{\cal E})$ from the undirected graph $G = (V,E)$.  Let $P$ be a permutation matrix which maps $V$ to $\C = \{C_1,C_2,\ldots,C_{l}\}$, i.e.,
\[ x({\C}) = P x(V) = P \xb \,. \]
The covariance of $x({\C})$ is 
\begin{align}
E[x({\C})x({\C})^T] &= P \Sigma P^T = [\Sigma_P(i,j)]_l \label{eq:cov_x_c}\\
E[x({C_i}) x({C_j})^T] &= \Sigma_P(i,j) \,, \text{ for  } i,j = 1,\ldots,l \,, \label{def:sigma_p}
\end{align}
where the notation $\textbf{F} = [F(i,j)]_l$ refers to the blocks of the matrix $\textbf{F}$ for $i,j = 1,\ldots,l$.  We now define some notation that will be useful in deriving the state-space representation and the recursive estimators.  The notation is borrowed from standard notation used for tree-structured graphs in \cite{ChouWillsky1994a,ChouWillsky1994b,WainwrightThesis}.  For the block-tree graph, we have defined $C_1$ as the root cluster.  The other clusters of the block-tree graph can be partially ordered according to their \emph{scale}, which is the distance of a cluster to the root cluster $C_1$.  This distance is defined as the number of edges in the path connecting two clusters.  Since we have a tree graph over the clusters, a path connecting two clusters is unique.  Thus, the scale of $C_1$ is zero.  All the neighbors of $C_1$ will have scale one.  For any $C_k$ at scale $s$, define $\{C_{\Gamma_1(k)},\ldots,C_{\Gamma_q(k)}\}$ as the set of clusters connected to $C_k$ at scale $s+1$ and $C_{\Upsilon(k)}$ as the cluster connected to $C_k$ at scale $s-1$.  Let ${\cal T}(C_k)$ be all the clusters including $C_k$ at scale greater than $s$ and at a distance greater than zero:
\begin{equation}
{\cal T}(C_k) = \{C_i : d(C_k,C_i) \ge 0 \text{ and } 
\text{scale}(C_i) > s\} \label{eq:rooted_cluster} \,,
\end{equation}
where scale$(C_i)$ is the scale of the cluster $C_i$. Fig.~\ref{fig:scale_tree}(a) shows the relevant notations on a block-tree graph.  Fig.~\ref{fig:scale_tree}(b) shows an example of a block-tree graph where the root cluster is $C_1 = \{7\}$, the cluster at scale one is $C_2 = \{7,8\}$, the clusters at scale two are $C_3 = \{1\}$ and $C_4 = \{5,9\}$, the cluster at scale three is $C_5 = \{2,6\}$, and the cluster at scale four is $C_6 = \{3\}$.  For $C_2$, we have that $C_{\Gamma_1(2)} = C_3$, $C_{\Gamma_2(2)} = C_4$, and $C_{\Upsilon(2)} = C_1$.  In the next Section, we derive a state-space representation for Gaussian random vectors defined on undirected graphs using block-tree graphs.

\subsection{State-Space Representation}
\label{sec:state_space}

Given a block-tree graph, we now derive state-space representations on the tree structure.  There are two types of representations we can define, one in which we are given $x\left(C_k\right)$ at scale $s$, and we want to compute $x\left(C_{\Upsilon(k)}\right)$ at scale $s-1$, and another in which we are given 
$x\left(C_{\Upsilon(k)}\right)$ at scale $s-1$, and we want to compute $x\left(C_k\right)$ at scale $s$.

\begin{theorem}
\label{thm:state_space}
Let $\xb = \{x_k \in \R : k \in V\}$ be a random vector defined over a graph $G = (V,E)$ and let ${\cal G} = (\C,{\cal E})$ be a block-tree graph obtained using a root cluster $C_1$.  We have the following linear representations:
{\singlespace
\vspace{-0.7cm}
\begin{align}
x(C_k) &= A_k x(C_{\Upsilon(k)}) + u(C_k) \label{eq:downward} \\ 
x(C_{\Upsilon(k)}) &= F_k x(C_k) + w(C_k) \,, \label{eq:upward}
\end{align}
}
where $u(C_k)$ is a \emph{white} Gaussian noise uncorrelated with $x\left(C_{\Upsilon(k)}\right)$, $w(C_k)$ is \emph{non-white} Gaussian noise uncorrelated with $x(C_k)$, and
{\singlespace
\vspace{-0.5cm}
\begin{align}
A_k &= \Sigma_P (C_k,C_{\Upsilon(k)})
\left[\Sigma_P(C_{\Upsilon(k)},C_{\Upsilon(k)})\right]^{-1} \\
Q_k^u &= E\left[ u(C_k) u^T(C_k) \right] = \Sigma_P(C_k,C_k) - A_k\Sigma_P(C_{\Upsilon(k)},C_k) \label{eq:q_u} \\
F_k &=  \Sigma_P (C_{\Upsilon(k)},C_k)\left[\Sigma_P(C_k,C_k)\right]^{-1} \\
Q_k^w &= E\left[ w(C_k) w^T(C_k) \right] = 
\Sigma_P(C_{\Upsilon(k)},C_{\Upsilon(k)}) - F_k \Sigma_P(C_k,C_{\Upsilon(k)})
\,,
\end{align}
}
where $\Sigma_P(i,j)$ is a block from the covariance of $x(\C)$ defined in (\ref{def:sigma_p}).
\end{theorem}
\begin{proof}
We first derive (\ref{eq:downward}).  Consider the cluster of nodes $\C \backslash {\cal T}(C_k)$.  From the global Markov property, stated in Definition \ref{def:global_markov_property}, we have
\begin{align}
\widehat{x}(C_k) &= E\left[x(C_k) | \{x(C_i) : C_i \in \C \backslash {\cal T}(C_k)\}\right]
= E\left[ x(C_k) | x(C_{\Upsilon(k)})\right] \,. \\
&= \Sigma_P (C_k,C_{\Upsilon(k)})
\left[\Sigma_P(C_{\Upsilon(k)},C_{\Upsilon(k)})\right]^{-1} x(C_{\Upsilon(k)}) = A_k x(C_{\Upsilon(k)}) \label{eq:t_1} \,,
\end{align}
where we get (\ref{eq:t_1}) using the Gauss-Markov theorem \cite{Scharf} for computing the conditional mean.  Define the error $u(C_k)$ as
\begin{equation}
u(C_k) = x(C_k) - \widehat{x}(C_k) = x(C_k) - A_k x(C_{\Upsilon(k)}) \,.
\end{equation}
It is clear that $u(C_k)$ is Gaussian.  Further, by the orthogonality properties of the minimum-mean squared error (mmse) estimates, $u(C_k)$ is white.  The variance of $u(C_k)$ is computed as follows:
{\singlespace
\vspace{-0.5cm}
\begin{align}
E\left[ u(C_k) u^T(C_k) \right] &= 
E\left[(x(C_k) -\widehat{x}(C_k))(x(C_k) - \widehat{x}(C_k))^T\right] \label{eq:t_2}\\
&=E\left[(x(C_k) -\widehat{x}(C_k))x^T(C_k)\right] \label{eq:t_3}\\ 
&=\Sigma_P(C_k,C_k) - A_k\Sigma_P(C_{\Upsilon(k)},C_k) \,.
\end{align}
}
To go from (\ref{eq:t_2}) to (\ref{eq:t_3}), we use the orthogonality of $u(C_k)$.  This gives us the $Q_k^u$ in (\ref{eq:q_u}).  Equation (\ref{eq:upward}) can be either derived in a similar manner or alternatively by using the results of \cite{VargheseKailath1979} on backwards Markovian models.
\end{proof}

The driving noise in (\ref{eq:upward}), $w(C_k)$, is not white noise.  This happens because for each $C_{\Upsilon(k)}$, there can be more than one cluster such that $C_{\Upsilon(j)} = C_{\Upsilon(k)}$.  Using the state-space representations, we can easily recover standard recursive algorithms, an example of which is shown in the next Section where we consider the problem of estimation over Gaussian graphical models.

\subsection{Recursive Estimation}
\label{sec:recursive_estimation}

Let $\xb \in \R^n$ be a zero mean Gaussian random vector defined on an undirected graph $G = (V,E)$.  Let $\Sigma$ be the covariance of $\xb$ and let $J = \Sigma^{-1}$.  Suppose we collect noisy observations of $\xb$ such that
\begin{equation}
y_s = H_s x_s + n_s \,, \label{eq:observations}
\end{equation}
where $n_s \sim {\cal N}(0,R_s)$ is white Gaussian noise independent of $x_s$ and $H_s$ is known.  Given $\yb = [y_1,\ldots,y_n]^T$, we want to find the minimum-mean squared error (mmse) estimate of $\xb$, which is $E[\xb | \yb]$.  From the Gauss-Markov theorem \cite{Scharf}, we have
\begin{align}
\widehat{x} &= E[\xb | \yb] = E[\xb \yb^T] 
\left(E[\yb \yb^T]\right)^{-1} \yb \\
&= \Sigma \left(H \Sigma H^T + R\right)^{-1} \yb \label{eq:mmse_estimate}\,,
\end{align}
where $H$ and $R$ are diagonal matrices with diagonals $H_s$ and $R_s$, respectively.  Using (\ref{eq:mmse_estimate}) to compute $\widehat{\xb}$ requires inversion of a $n \times n$ matrix, which has complexity $O(n^3)$.  An alternate method is to use the state-space representations in Theorem \ref{thm:state_space} and derive standard Kalman filters and recursive smoothers using \cite{ChouWillsky1994a,ChouWillsky1994b} \footnote{The results in \cite{ChouWillsky1994a} are on dyadic trees, however they can be easily generalized to arbitrary trees.}.  This approach will require inversion of a $\text{btw}(G) \times \text{btw}(G)$ matrix, where $\text{btw}(G)$ is the treewidth of the graph.  Another approach to computing $\widehat{\xb}$ is to use the equations in \cite{SudderthThesis2002}, where they derive estimation equations for Gaussian tree distributions given $J = \Sigma^{-1}$.  The generalization to block-tree graphs is trivial and will only involve identifying appropriate blocks from the inverse of the covariance matrix.  Thus, by converting an arbitrary graph into a block-tree graph, we are able to recover algorithms for recursive estimation of Gaussian graphical models.
For graphs with high block-treewidth, however, computing mmse estimates is computationally intractable.  For this reason, we propose an efficient approximate estimation algorithm in the next Section.

\section{Approximate Estimation}
\label{sec:approximate_estimation}

In this Section, we use the block-tree graph framework to derive approximate estimation algorithms for Gaussian graphical models.  The need for approximate estimation arises because estimation/inference in graphical models is computationally intractable for graphical models with large treewidth or large block-treewidth.  The approach we use for approximate estimation is based on decomposing the original graphical model into computationally tractable subgraphs and using the subgraphs for estimation, see \cite{WainwrightMAP2002,WainwrightTRP2003,SudderthWainwrightWillsky2004} for a general class of algorithms.
Traditional approaches to finding subgraphs involve using spanning trees, which are tree-structured subgraphs.  We propose to use 
\emph{spanning block-trees}, which are block-tree graphs with low block-treewidth.  Section \ref{sec:spanning_bt} outlines a heuristic algorithm for finding maximum weight spanning block-trees.  
We review the matrix splitting approach to approximate estimation in Section \ref{sub:estimation_matrix_splitting} and show how spanning block-trees can be used instead of spanning trees.  Section \ref{sub:experimental_results} shows experimental results.

\subsection{Spanning Block-Trees}
\label{sec:spanning_bt}

Let $G = (V,E)$ be an undirected graph.  We define a $B$-width spanning block-tree as a subgraph of $G$ with block-treewidth at most $B$.  If $B=1$, the spanning block-tree becomes a spanning tree.  For $B > 1$, we want to remove edges from the graph $G$ until we get a block-tree graph with block-width less than or equal to $B$.  To quantify each edge, we associate a weight $w_{i,j}$ for each $(i,j) \in E$.  If $w_{i,j}$ is the mutual information between nodes $i$ and $j$, finding an optimal spanning block-tree by removing edges reduces to minimizing the Bethe free energy \cite{Yedidia2000,Chechetka2009}.  For the purpose of approximate estimation in Gaussian graphs, the authors in \cite{VenkatJohnsonWillsky2008} proposed weights that provided a measure of error-reduction capacity of each edge in the graph.  We use these weights when finding spanning block-trees for the purpose of approximate estimation in Section \ref{sub:estimation_matrix_splitting}.  If a graph is not weighted, we can assign all the weights to be one.  In this case, finding a maximum weight $B$-width spanning block-tree is equivalent to finding a $B$-width spanning block-tree which retains the most number of edges in the final subgraph.

{ \singlespace
\begin{algorithm}
\caption{Constructing Maximum Weight Spanning Block-Trees}\label{alg:spanningblocktree}
\begin{algorithmic}[1]
\Procedure{MWSpanningBlockTree}{$G,W,B$}
\State ${\cal G}$ = FindOptimalBlockTree(G) ; ${\cal G} = ({\cal C},{\cal E})$ \label{line:optimal_bt}
\State ${\cal C}^B \leftarrow$ SplitCluster(${\cal G},W$); See Section \ref{sub:splitting_clusters} \label{line:splitting}
\State ${\cal G}^B \leftarrow $ MWST$\left(\{C_{k_i}^{l_k}\},W\right)$; 
See Section \ref{sub:find_bt_from_clusters} \label{line_connect}
\EndProcedure
\end{algorithmic}
\end{algorithm}
}

Algorithm \ref{alg:spanningblocktree} outlines our approach to finding maximum weight spanning block-trees.  The input to the algorithm is an undirected graph $G$, a weight matrix $W = [w_{i,j}]_n$, and the desired block-treewidth $B$.  The output of the algorithm is a block-tree graph 
${\cal G}^B = ({\cal C}^B,{\cal E}^B)$, where $\text{btw}({\cal G}) \le B$.
Algorithm \ref{alg:spanningblocktree} is a greedy algorithm for finding spanning block-trees since solving this problem optimally is combinatorially complex.  We first find the optimal block-tree graph ${\cal G}$ for the undirected graph $G$ using the algorithm outlined in Section~\ref{sec:optimal_block_tree} (Line \ref{line:optimal_bt}).  The next steps in the algorithm are: (i) Splitting ${\cal C}$ into clusters ${\cal C}^B$ (Line \ref{line:splitting}), and (ii) finding edges ${\cal E}^B$ connecting two clusters so that ${\cal G}^B = ({\cal C}^B,{\cal E}^B)$ is a block-tree graph (Line \ref{line_connect}).

\bigskip
\subsubsection{\textbf{Splitting Clusters}}
\label{sub:splitting_clusters}
Since the maximum size of the cluster in ${\cal C}$ is greater than $B$, we first identify all clusters $\{C_{k_1},\ldots,C_{k_m}\}$ such that 
$|C_{k_i}| > B$ for all $i = 1,\ldots,m$.  Next, we split each $C_{k_i}$ into smaller clusters so that $C_{k_i} = \{C_{k_i}^1,C_{k_i}^2,\ldots,C_{k_i}^{h}\}$, where $|C_{k_i}^{j}| \le B$ for $j = 1,\ldots,h$.  This splitting must be done in such a way the the nodes in the same cluster retain edges in the original graph with maximum weight.  The algorithm we propose for splitting the clusters is as follows:
\begin{enumerate}[a)]
\item For each $C_{k_i}$, let 
$\Gamma(C_{k_i}) = \{C_{\Gamma_1(k_i)} , C_{\Gamma_2(k_i)}, \ldots, C_{\Gamma_q(k_i)}\}$ be all the clusters connected to $C_{k_i}$ at the next scale and let $C_{\Upsilon({k_i})}$ be the cluster connected to $C_{k_i}$ at the previous scale.  We assume that we have already split $C_{\Upsilon({k_i})}$ such that $C_{\Upsilon({k_i})} = \{C_{\Upsilon({k_i})}^1,\ldots,C_{\Upsilon({k_i})}^{w}\}$, where $|C_{\Upsilon({k_i})}^j| \le B$, for $j = 1,\ldots,w$.  The notations introduced are shown in Fig. \ref{fig:splitting_notation}(a).

\begin{figure}
\begin{center}
\scalebox{0.8} 
{
\begin{pspicture}(0,-2.5684373)(15.120637,2.5484374)
\definecolor{color821b}{rgb}{0.6,0.6,1.0}
\psellipse[linewidth=0.04,dimen=outer,fillstyle=solid,fillcolor=color821b](14.790637,-1.4715625)(0.33,0.28)
\psellipse[linewidth=0.04,dimen=outer,fillstyle=solid,fillcolor=color821b](13.160146,-1.4615625)(1.1995087,0.37000006)
\psellipse[linewidth=0.04,dimen=outer,fillstyle=solid,fillcolor=color821b](7.0201454,-1.5015625)(1.5595087,0.41000006)
\psellipse[linewidth=0.04,dimen=outer,fillstyle=solid,fillcolor=color821b](7.0001454,-0.0615625)(1.5595087,0.41000006)
\psdots[dotsize=0.2](6.220637,-1.4915625)
\psdots[dotsize=0.2](6.600637,-1.4915625)
\psdots[dotsize=0.2](7.020637,-1.4915625)
\psdots[dotsize=0.2](7.420637,-1.4915625)
\psdots[dotsize=0.2](7.820637,-1.4915625)
\psline[linewidth=0.04cm](5.820637,-0.11156245)(6.200637,-1.4115624)
\psline[linewidth=0.04cm](6.580637,-0.11156245)(6.240637,-1.4315624)
\psline[linewidth=0.04cm](7.400637,-0.07156245)(6.600637,-1.4515624)
\psline[linewidth=0.04cm](7.420637,-0.11156245)(7.020637,-1.4315624)
\psline[linewidth=0.04cm](8.240638,-0.07156245)(7.820637,-1.4315624)
\psline[linewidth=0.04cm](6.620637,-0.09156245)(6.600637,-1.4315624)
\psline[linewidth=0.04cm](7.440637,-0.09156245)(7.420637,-1.4315624)
\psline[linewidth=0.04cm](8.240638,-0.03156245)(7.440637,-1.4515624)
\psline[linewidth=0.04cm](5.860637,-0.09156245)(6.580637,-1.4715625)
\psline[linewidth=0.04cm](6.620637,-0.09156245)(7.000637,-1.4515624)
\psline[linewidth=0.04cm](6.640637,-0.07156245)(7.8006372,-1.4315624)
\psline[linewidth=0.04cm](6.660637,-0.01156245)(7.420637,-0.03156245)
\rput{89.45386}(0.543172,-0.93248415){\psellipse[linewidth=0.04,dimen=outer](0.7422936,-0.1920549)(1.5794517,0.50201976)}
\rput{89.45386}(2.02185,-2.734203){\psellipse[linewidth=0.04,dimen=outer](2.3911202,-0.34649423)(1.2949638,0.4291644)}
\rput{0.13007866}(0.0015653649,-0.008992194){\psellipse[linewidth=0.04,dimen=outer](3.961575,0.68500024)(0.52,0.26)}
\rput{89.45386}(-0.57989985,-2.0278502){\psellipse[linewidth=0.04,dimen=outer](0.7336862,-1.3066521)(0.32,0.16)}
\rput{89.45386}(0.16701916,-1.3019902){\psellipse[linewidth=0.04,dimen=outer](0.74073964,-0.5666857)(0.32,0.16)}
\psdots[dotsize=0.08,dotangle=89.45386](0.7462681,0.013287968)
\psdots[dotsize=0.08,dotangle=89.45386](0.7481744,0.21327889)
\psdots[dotsize=0.08,dotangle=89.45386](0.7500808,0.4132698)
\psdots[dotsize=0.08,dotangle=89.45386](3.95109,-0.17494975)
\psdots[dotsize=0.08,dotangle=89.45386](3.9529965,0.025041156)
\psdots[dotsize=0.08,dotangle=89.45386](3.954903,0.22503208)
\psdots[dotsize=0.12,dotangle=89.45386](2.3366637,-1.0019176)
\psdots[dotsize=0.12,dotangle=89.45386](2.3404763,-0.6019356)
\psdots[dotsize=0.12,dotangle=89.45386](2.3440983,-0.22195292)
\psdots[dotsize=0.12,dotangle=89.45386](2.3479111,0.17802893)
\psdots[dotsize=0.12,dotangle=89.45386](2.3519144,0.5980098)
\psline[linewidth=0.04cm](0.89367884,-1.3081771)(2.2764757,-1.0213447)
\psline[linewidth=0.04cm](2.3119164,0.5983912)(0.8813052,-0.50802284)
\psline[linewidth=0.04cm](2.3002875,-0.6215534)(0.9003511,-0.60820895)
\usefont{T1}{ptm}{m}{n}
\rput{-1.0}(0.031648107,0.02580167){\rput(1.4750508,-1.7806687){$C_{\Upsilon(k_i)}$}}
\usefont{T1}{ptm}{m}{n}
\rput{-0.09644491}(0.0030078546,0.0050609815){\rput(2.9890575,-1.7646812){$C_{k_i}$}}
\psline[linewidth=0.04cm](0.87501425,-1.1679928)(2.3039095,-0.24157071)
\psline[linewidth=0.04cm](0.9211127,-0.5284032)(2.3077223,0.15841112)
\usefont{T1}{ptm}{m}{n}
\rput(2.1909494,-2.2615623){(a)}
\psdots[dotsize=0.2,linecolor=red](5.840637,-0.01156255)
\psdots[dotsize=0.2,linecolor=red](6.6406364,-0.01156255)
\psdots[dotsize=0.2,linecolor=red](7.440637,-0.01156255)
\psdots[dotsize=0.2,linecolor=red](8.260637,-0.01156245)
\psdots[dotsize=0.2](9.020637,-0.5515624)
\psdots[dotsize=0.2](9.820637,-0.5515624)
\psdots[dotsize=0.2](10.620637,-0.5515624)
\psdots[dotsize=0.2](11.440637,-0.5515624)
\psline[linewidth=0.04cm](11.380637,2.5084374)(11.380637,2.5284374)
\psline[linewidth=0.04cm](9.040637,-0.53156245)(11.460638,-0.53156245)
\usefont{T1}{ptm}{m}{n}
\rput(9.467824,-0.34156245){2}
\usefont{T1}{ptm}{m}{n}
\rput(10.185949,-0.34156245){3}
\psbezier[linewidth=0.04](9.020637,-0.51156247)(9.020637,-1.3115624)(10.620637,-1.3515625)(10.620637,-0.5515624)
\usefont{T1}{ptm}{m}{n}
\rput(9.644386,-1.3415625){1}
\usefont{T1}{ptm}{m}{n}
\rput(10.984386,-0.34156245){1}
\psbezier[linewidth=0.04](9.820637,-0.51156247)(9.820637,-1.3115624)(11.420638,-1.3115624)(11.420638,-0.51156247)
\usefont{T1}{ptm}{m}{n}
\rput(10.724387,-1.3415625){1}
\rput{0.13007866}(-0.0014768356,-0.008995648){\psellipse[linewidth=0.04,dimen=outer](3.961575,-0.6549997)(0.52,0.26)}
\rput{0.13007866}(-0.0032930234,-0.009043116){\psellipse[linewidth=0.04,dimen=outer](3.9815753,-1.4549998)(0.52,0.26)}
\usefont{T1}{ptm}{m}{n}
\rput(2.5420432,-1.0015626){$r$}
\rput{89.45386}(1.6269528,0.14409345){\psellipse[linewidth=0.04,dimen=outer](0.74073964,0.8933143)(0.32,0.16)}
\usefont{T1}{ptm}{m}{n}
\rput(2.5657933,-0.60156244){$r'$}
\usefont{T1}{ptm}{m}{n}
\rput(2.5170434,-0.20156245){$s$}
\psline[linewidth=0.04cm](0.90063703,0.92843753)(2.340637,0.6284376)
\psline[linewidth=0.04cm](2.720637,-1.0315624)(3.480637,-1.4115624)
\psline[linewidth=0.04cm](2.780637,-0.53156245)(3.460637,-0.61156243)
\psline[linewidth=0.04cm](2.700637,0.5084376)(3.460637,0.66843754)
\psellipse[linewidth=0.04,dimen=outer,fillstyle=solid,fillcolor=color821b](12.370638,-0.39156246)(0.33,0.28)
\psellipse[linewidth=0.04,dimen=outer,fillstyle=solid,fillcolor=color821b](14.790637,-0.41156244)(0.33,0.28)
\psellipse[linewidth=0.04,dimen=outer,fillstyle=solid,fillcolor=color821b](13.590636,-0.40156245)(0.73,0.37)
\usefont{T1}{ptm}{m}{n}
\rput(7.0209494,-2.3015625){(b)}
\psdots[dotsize=0.2,linecolor=red](11.440638,-0.5515624)
\psdots[dotsize=0.2,linecolor=red](10.620637,-0.55156255)
\psdots[dotsize=0.2,linecolor=red](9.820637,-0.55156255)
\psdots[dotsize=0.2,linecolor=red](9.0206375,-0.55156255)
\psdots[dotsize=0.2,linecolor=red](14.800637,-0.41156244)
\psdots[dotsize=0.2,linecolor=red](13.980638,-0.41156256)
\psdots[dotsize=0.2,linecolor=red](13.180636,-0.41156256)
\psline[linewidth=0.04cm](13.259654,-0.41156256)(13.880637,-0.41156244)
\psdots[dotsize=0.2,linecolor=red](12.380637,-0.41156256)
\psdots[dotsize=0.2](12.380636,-1.4715625)
\psdots[dotsize=0.2](13.180636,-1.4715625)
\psdots[dotsize=0.2](13.980637,-1.4715625)
\psdots[dotsize=0.2](14.800636,-1.4715625)
\psdots[dotsize=0.2,linecolor=red](14.800637,-1.4715625)
\psdots[dotsize=0.2,linecolor=red](13.980638,-1.4715625)
\psdots[dotsize=0.2,linecolor=red](13.180636,-1.4715625)
\psdots[dotsize=0.2,linecolor=red](12.380637,-1.4715625)
\psline[linewidth=0.04cm](13.259654,-1.4715625)(13.880637,-1.4715625)
\usefont{T1}{ptm}{m}{n}
\rput(10.17095,-2.3215625){(c)}
\usefont{T1}{ptm}{m}{n}
\rput(13.520949,-2.3415625){(d)}
\end{pspicture} 
}
\caption{(a) Notation used in Section \ref{sub:splitting_clusters}.
(b) An example showing how clusters are split into smaller clusters.  
The goal is to split the cluster with red nodes into clusters with maximum size $2$ and $3$.  All edges are assumed to have weight one. 
(c) Weighted graph constructed using $\eta_{rs}$ given in (\ref{eq:weights_split}). 
(d) Smaller clusters formed for $B = 2$ and $B = 3$.  For $B = 2$, we first choose the middle two nodes since the weight between these two nodes in (c) is maximum.  The remaining nodes are unconnected in (c), so we assign them to individual clusters.  For $B = 3$, we again choose the middle two clusters first and then add another cluster so that the sum of weights is maximal. }
\label{fig:splitting_notation}
\end{center}
\end{figure}
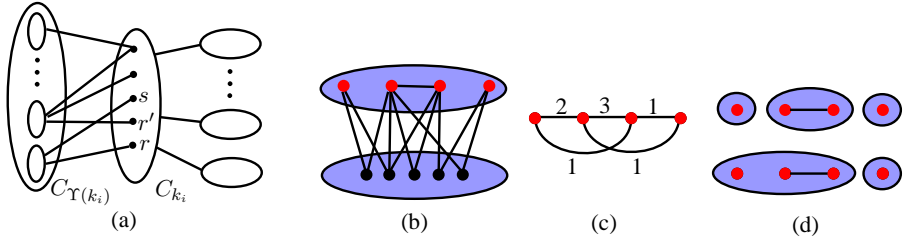

\item To split the cluster $C_{k_i}$, we first identify nodes in $C_{k_i}$ which can be clustered together.  For any two distinct nodes $r,s \in C_{k_i}$, if both $r$ and $s$ have edges in any one of the clusters 
$C_{\Upsilon({k_i})}^j$ for any $j = 1,\ldots,h$, then nodes $r$ and $s$ can be clustered together.  For example, in Fig. \ref{fig:splitting_notation}(a), nodes $r$ and $r'$ can not be clustered together, whereas nodes $r$ and $s$ can be clustered together.

\item We now associate weights $\eta_{rs}$ between nodes of $C_{k_i}$ that can be clustered together:
\begin{equation}
\eta_{rs} = 
w_{rs} + \sum_{t \in {\cal N}(r) \cap {\cal N}(s) \cap \Gamma(C_{k_i})}
\left( w_{rt} + w_{ts} \right) \label{eq:weights_split} \,.
\end{equation}
The intuition behind constructing the weights in (\ref{eq:weights_split}) is to cluster nodes together that are connected to the same cluster at the next scale or are connected to each other.  An example of constructing $\eta_{rs}$ is shown in Fig \ref{fig:splitting_notation}(b) and Fig \ref{fig:splitting_notation}(c), where we assume the weights on each edge are one.  

\item Using $\eta_{rs}$ we construct a weighted graph on the nodes in $C_{k_i}$.  To construct smaller clusters, we first choose two nodes for which $\eta_{rs}$ is maximum.  We keep adding nodes to this cluster by choosing nodes connected to at least one node in this cluster until the cluster size is $B$.  If no other node is connected to the new cluster, we start building another cluster.  For example, given the weighted graph in Fig.~\ref{fig:splitting_notation}(c), we construct clusters in 
Fig.~\ref{fig:splitting_notation}(d) with $B = 2$ and $B = 3$.
\end{enumerate}
After applying Steps (a)-(d) on all $C_{k_i}$ such that $|C_{k_i}| > B$, we get the news clusters ${\cal C}^B$.

\input{spanning_block.tex}

\bigskip
\subsubsection{\textbf{Find block-tree graph from clusters}}
\label{sub:find_bt_from_clusters}
Given the clusters ${\cal C}^B = \{C_k^B\}$, we can find a block-graph and associate weights between the clusters in ${\cal C}^B$ such that
\begin{equation}
w^B_{i,j} = \sum_{(i,j) \in (C_i^B \times C_j^B) \cap E} w_{i,j} \,. \label{eq:w_c}
\end{equation}`
Equation (\ref{eq:w_c}) corresponds to the sum of all weights connecting two clusters $C_i^B$ and $C_j^B$.  Given the weights in (\ref{eq:w_c}), we can easily find a spanning block-tree using the maximum weight spanning tree (MWST) algorithms of Prim \cite{Prim1957} or Kruskal \cite{Kruskal1956}.  Fig. \ref{fig:spanning_blocktree} shows an example of using Algorithm 2 to find a maximum weight $2$-width spanning block-tree. Fig. \ref{fig:spanning_grid} shows a collection of spanning block-trees for a $4 \times 4$ grid graph.

\bigskip
\subsubsection{\textbf{Complexity}}
The complexity of Algorithm 2 depends on the structure of the graph.  
Assuming the optimal block-tree graph is given, the complexity of splitting the clusters is $O(|C_{k_i}|^2)$ for each $C_{k_i}$ such that $|C_{k_i}| > B$.  This number will be dominated by the cluster with maximum size, thus the complexity of splitting clusters is $O((\text{btw}(G))^2)$.  The complexity of finding the final spanning block-tree depends on the graph and the number of edges in the block-graph formed using the clusters ${\cal C}^B$.  In general, the complexity of this step decreases as $B$ increases since this results in less number of edges and less number of smaller clusters ${\cal C}^B$.
In practice, finding the optimal block-tree graph is hard, so we use the heuristic algorithm outlined in Section \ref{sec:optimal_block_tree}.

\subsection{Estimation Via Matrix Splitting}
\label{sub:estimation_matrix_splitting}
This Section reviews the matrix splitting approach to approximate estimation in Gaussian graphical models.  For more details, see  \cite{SudderthWainwrightWillsky2004} and \cite{VenkatJohnsonWillsky2008}.
Let $\xb \in \R^n$ be a Gaussian graphical model defined on a graph $G = (V,E)$ with covariance $\Sigma$ and observations given by (\ref{eq:observations}).  The mmse estimate is given in (\ref{eq:mmse_estimate}).  An alternate characterization of (\ref{eq:mmse_estimate}) is in the information form \cite{SudderthWainwrightWillsky2004}:
\begin{align}
V \widehat{\xb} &= H^T R^{-1} \yb \\ \label{eq:mmse_in}
V &= (J + H^T R^{-1} H) \,,
\end{align}
where $J = \Sigma^{-1}$ and $\widehat{P} = V^{-1}$ is the error covariance matrix.  The matrices $H$ and $R$ are assumed to be diagonal, so the sparsity of $V$ is the same as the sparsity of $J$.  A family of approximate estimation algorithms, which are iterative algorithms, have been proposed in \cite{SudderthWainwrightWillsky2004}, with extensions in \cite{DelouilleNeelamanuBaranuik2006, VenkatJohnsonWillsky2008}.
The idea is to split the matrix $V$ at each iteration $k$ as $V = V_{S_k} - K_{S_k}$, where $S_k$ is a subgraph of $G$.  The sparsity of $V_{S_k}$ corresponds to the sparsity of the subgraph $S_k$ and the diagonals of $V_{S_k}$ are the same as the diagonals $V$.  Using matrix splitting, an iterative algorithm for estimation is given as \cite{SudderthWainwrightWillsky2004}:
\begin{equation}
V_{S_k} \widehat{\xb}^{(k)} = K_{S_k} \widehat{\xb}^{(k-1)} + H^T R^{-1} \yb \,,
\label{eq:et}
\end{equation}
where $\widehat{\xb}^{(k)}$ is estimate at step $k$.  If $S_k$ is a subgraph with low treewidth or low block-treewidth, computing (\ref{eq:et}) is computationally tractable.  Conditions for convergence of (\ref{eq:et}) are not known for general graphical models, however for walk-summable graphical models \cite{MalioutovJohnsonWillsky2006}, convergence is guaranteed \cite{VenkatJohnsonWillsky2008}.  To compute the error covariance $\widehat{P}$, we can use the same matrix splitting approach to solve the linear system $V \widehat{P} = I_n$, where $I_n$ is an $n \times n$ identity matrix \cite{SudderthWainwrightWillsky2004}.

It is clear that the choice of $S_k$ in (\ref{eq:et}) plays an important role in the convergence of the algorithm.
The problem of adaptively choosing $S_k$ at each iteration was considered in \cite{VenkatJohnsonWillsky2008,VenkatJohnsonWillsky2008c}, where the authors proposed a weight matrix $w_{u,v}$, for $(u,v) \in E$, which signified the error reduction capacity of an edge $(u,v)$ in the iterations (\ref{eq:et}):
\begin{align}
w_{u,v}^{(k)} &= \left( |h_u^{(k-1)}| + |h_v^{(k-1)}| \right) \frac{|J(u,v)|}{1 - |J(u,v)|} \,, \label{eq:weights_estimation} \\ 
h_u^{(k-1)} &= H^T R^{-1} \yb - V \widehat{\xb}^{(k-1)} \label{eq:error_estimation}\,.
\end{align}
Thus, at each iteration we want to choose a subgraph $S_k$ such that the sum of all the weights in the graph is maximized while $S_k$ is still a tractable subgraph.  A popular approach is to use spanning trees, since finding an optimal spanning tree is efficient.  However, as shown in \cite{VenkatJohnsonWillsky2008c}, using tractable subgraphs, which are not trees, leads to faster convergence.  Motivated by the need for algorithms with faster convergence, we propose to use spanning block-trees for approximate estimation of Gaussian graphical models.  Thus, at each iteration we compute a weighted graph using (\ref{eq:weights_estimation}) and then use these weights to compute a $B$-width spanning block-tree using Algorithm \ref{alg:spanningblocktree}.  In the next Section, we provide experimental results and show the improved convergence rates of estimation when using spanning block-tree graphs over spanning trees.

\subsection{Experimental Results}
\label{sub:experimental_results}

Let $G = (V,E)$ be an undirected graph and suppose $\xb$ is a Gaussian graphical model defined on $G$ with covariance $\Sigma$ and $J = \Sigma^{-1}$.  We assume the diagonals of $J$ are unity and let $S = I - J$.  The non-zero entries in $H$ correspond to the edges in the graph.  To construct Gaussian graphical models, we choose the non-zero entries in $S$ uniformly between $[-1,1]$ and rescale $S$ so that $\rho(\bar{S}) = 0.99$, where $\bar{S}$ is the matrix of absolute values of the elements of $R$ and $\rho(\cdot)$ is the spectral radius.  From \cite{MalioutovJohnsonWillsky2006}, $\rho(\bar{S}) < 1$ ensures that the graphical model is walk-summable, which in turn ensures convergence of the iterative approximate estimation algorithm in (\ref{eq:et}) \cite{VenkatJohnsonWillsky2008}.  In all experiments, we assume that the observations are given by $y_s = x_s + n_s$, where $n_s \sim {\cal N}(0,10)$.  At each iteration, we compute the residual error, defined as $\frac{||{\bf h}^{(n)}||^2}{||{\bf h}^{(0)}||^2}$, where ${\bf h}^{(n)} = [h_1^{(n)},\ldots,h_n^{(n)}]^T$, for $h_u^{(n)}$ defined in (\ref{eq:error_estimation}).

Fig. \ref{fig:grids}(a) shows results of doing estimation over a randomly generated $50 \times 50$ grid graph using spanning trees (Tree), spanning block-trees with block-width of three (BT-3), and spanning block-trees with block-width of five (BT-5).  It is clear that using spanning block-trees leads to faster convergence.  The same results hold for Fig. \ref{fig:grids}(b) that shows results on doing estimation over a randomly generated $70 \times 70$ grid graph.  

Fig. \ref{fig:15x15grid} shows results of doing estimation over a randomly generated $15 \times 15$ grid graph where two nodes are connected to all the nodes in the graph.  Such graphs, where a few nodes have very high degree, are useful in video surveillance, modeling air traffic routes using hub-and-spoke model, or applications using small-world graphs.  Fig. \ref{fig:15x15grid}(a) plots the residual error at each iteration for the estimate and Fig. \ref{fig:15x15grid}(b) plots the residual error at each iteration for the error covariance.  Again, we observe that using spanning block-trees leads to faster convergence.
The above simulations show that using spanning block-trees for approximate estimation is viable and leads to improved convergence speed.

\begin{figure}
\begin{center}
\subfigure[Estimating $\widehat{\xb}$ on a $50 \times 50$ grid graph.]{
\label{fig:30x30gride}
\includegraphics[scale=0.4]{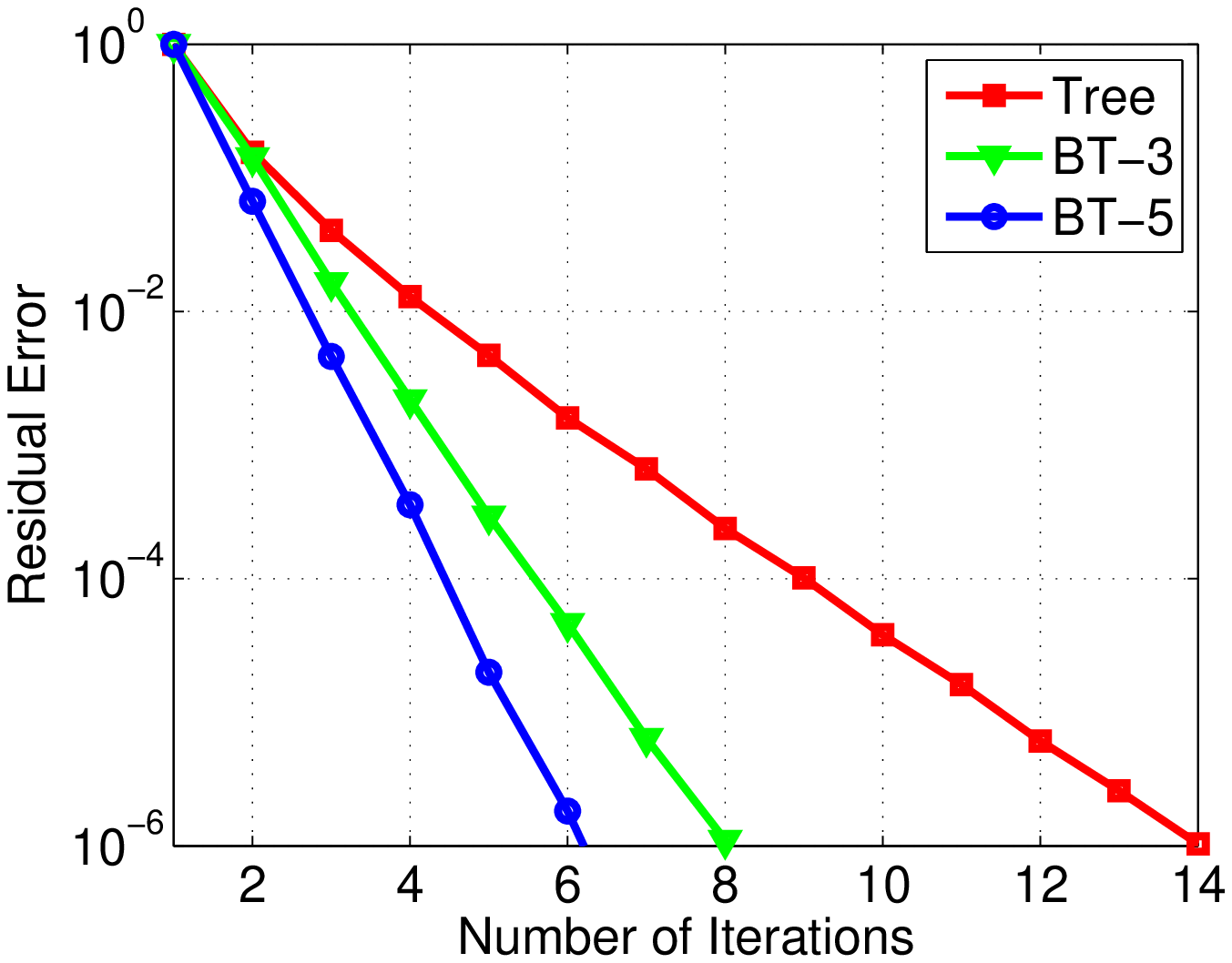}
}
\subfigure[Estimating $\widehat{\xb}$ on a $70 \times 70$ grid graph.]{
\label{fig:50x50gride}
\includegraphics[scale=0.4]{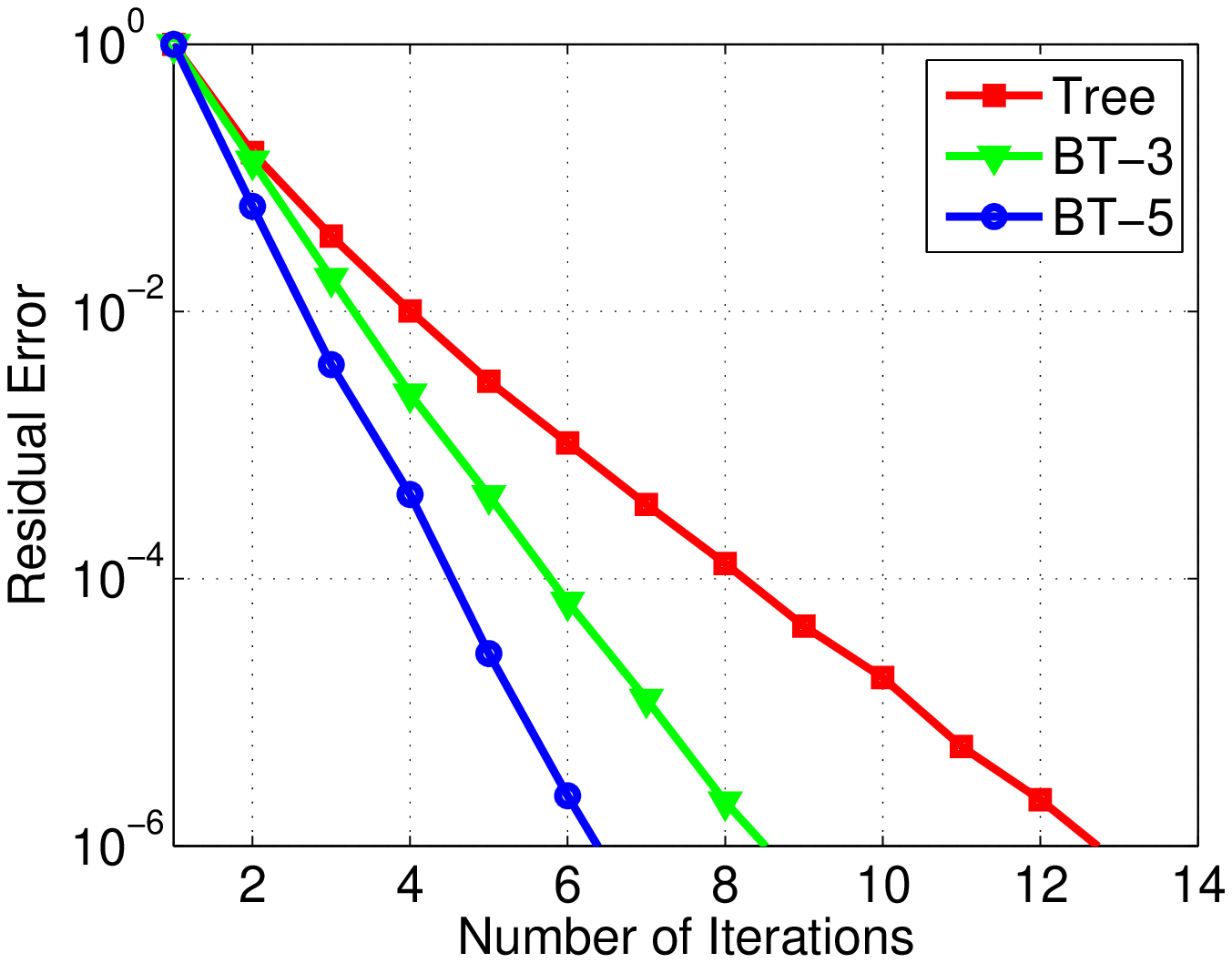}
}
\caption{Normalized residual error when doing approximate estimation on a $50 \times 50$ grid graph and a $70 \times 70$ grid graph.}
\label{fig:grids}
\end{center}
\vspace{-0.3cm}
\end{figure}

\begin{figure}
\begin{center}
\subfigure[Estimating $\widehat{\xb}$]{
\label{fig:15x15gride}
\includegraphics[scale=0.4]{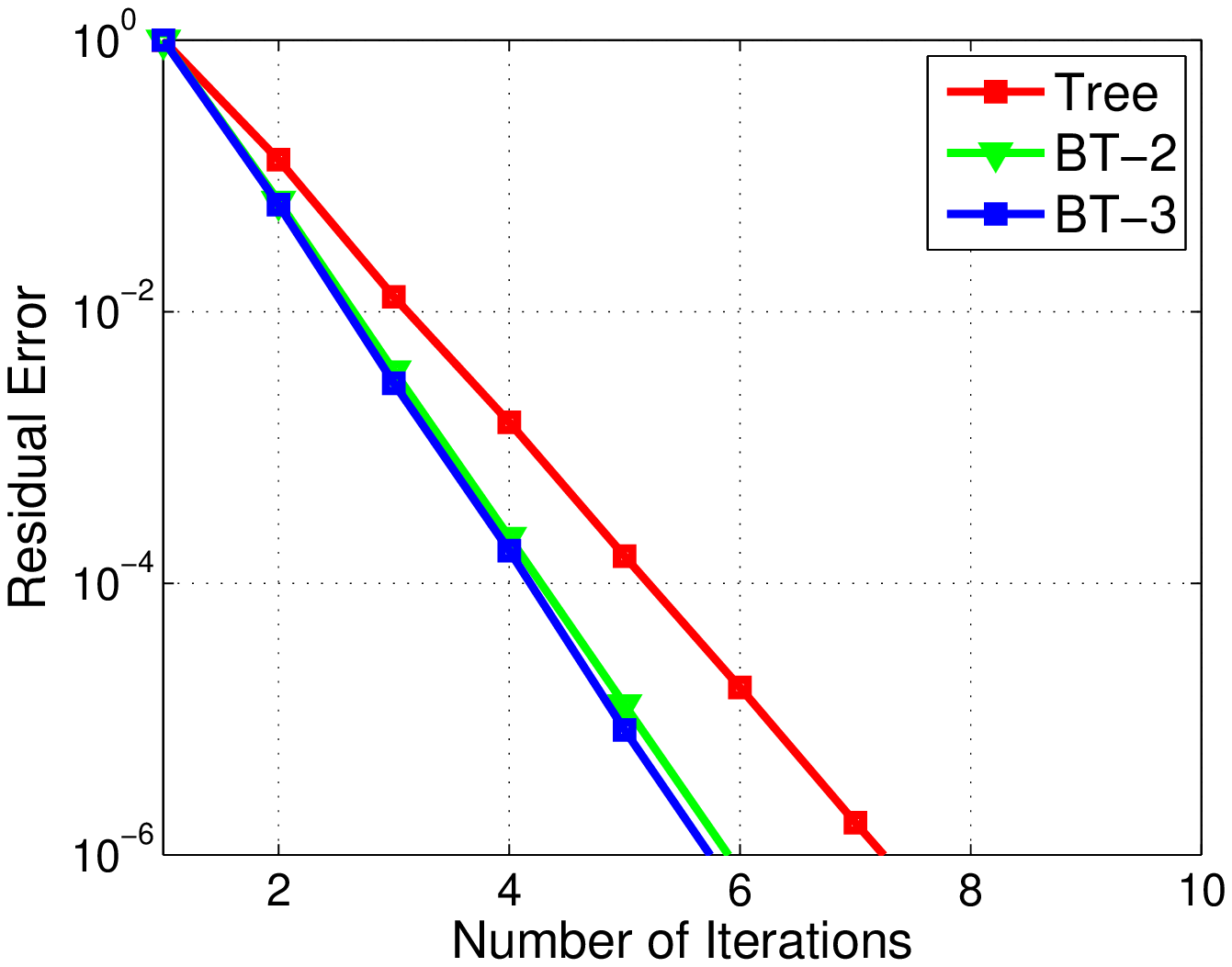}
}
\subfigure[Estimating $\widehat{P}$]{
\label{fig:15x15gridp}
\includegraphics[scale=0.4]{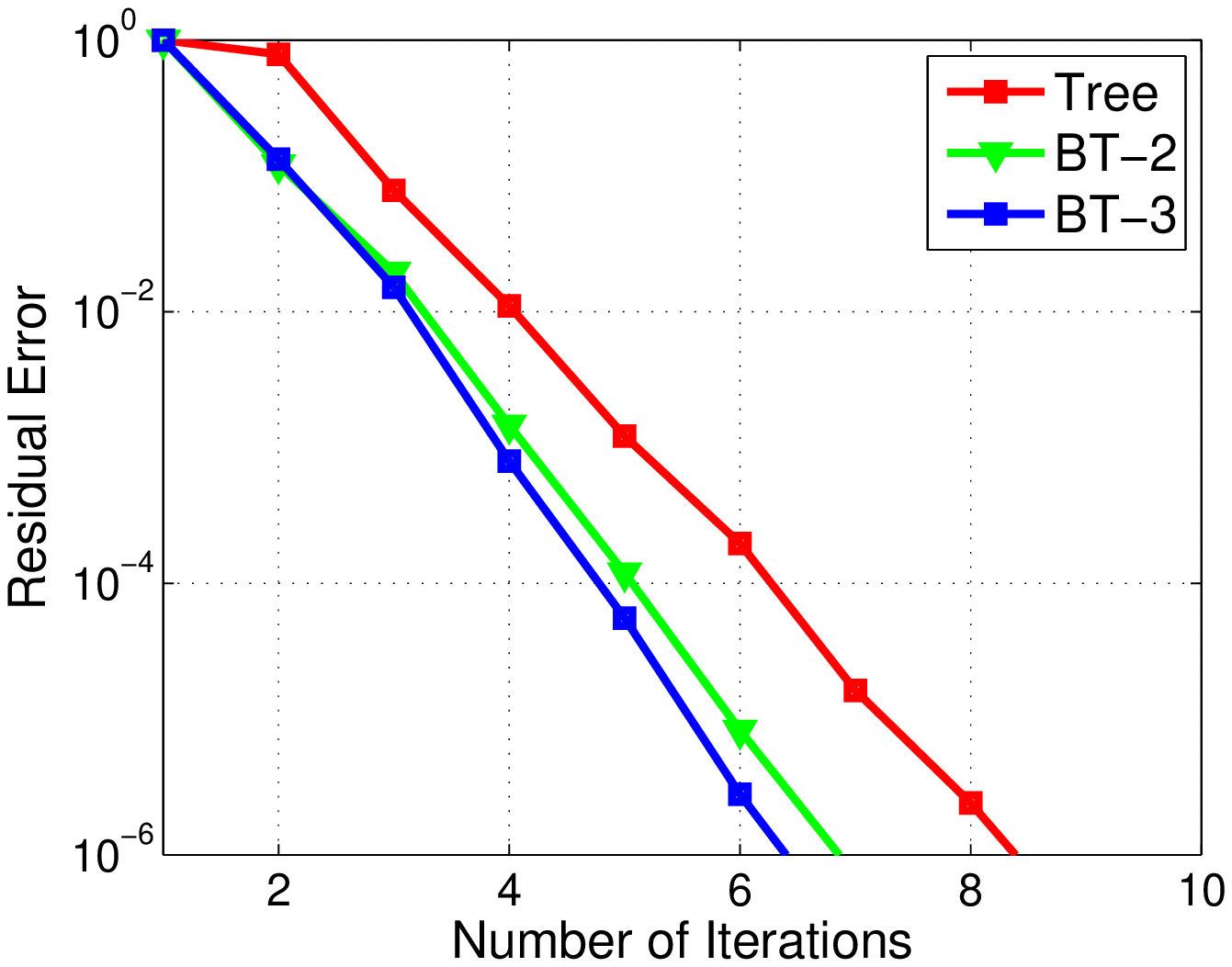}
}
\end{center}
\caption{Normalized residual error when doing approximate estimation using spanning trees, spanning block-trees with block-width of two (BT-2), and spanning block-trees with block-width of three (BT-3) on a $15 \times 15$ grid graph with two nodes connected to all other nodes in the graph.}
\label{fig:15x15grid}
\vspace{-0.3cm}
\end{figure}

\section{Summary}
\label{sec:summary}

We introduced block-tree graphs as an alternative to junction-trees for constructing tree-structured graphs for arbitrary graphical models.  We showed that constructing block-tree graphs is simple and only requires information about a root cluster, which is a small number of nodes.  On the other hand, constructing junction-trees requires knowledge of almost all nodes in the graph.
For graphical models with boundary conditions, we showed that the block-tree graph framework leads to natural representations where we converted a boundary valued problem into a initial value problem.  For Gaussian graphical models, the block-tree graph framework leads to state-space representations, using which we can easily recover recursive algorithms.  Using the block-tree graph framework, we derived an algorithm for approximate estimation in Gaussian graphical models.  The need for such algorithms arises because the problem of exact optimal estimation is computationally intractable for graphs with high treewidth.  We proposed the use of spanning block-trees to derive approximate estimation algorithms for Gaussian graphical models.  We showed that the speed of convergence when using spanning block-trees is faster when compared to using spanning trees.  Further applications of spanning block-trees can be explored when doing approximate inference over discrete graphical models, using the results of \cite{WainwrightMAP2002,Kolmogorov06}.

{\singlespace

}

\end{document}